\documentclass[9pt,twocolumn,twoside]{pnas-new}

\templatetype{pnasresearcharticle} 

\newcommand{\bigO}[1]{\mathcal{O}(#1)}

\newcommand{\mc}[1]{\mathcal{#1}}

\newcommand{\vts}[1]{\lvert #1 \rvert}
\newcommand{\Vts}[1]{\lVert #1 \rVert}
\newcommand{\bb}[1]{\mathbb{#1}}
\newcommand\inv[1]{#1\raisebox{1.05ex}{$\scriptscriptstyle-\!1$}}
\newcommand\Tstrut{\rule{0pt}{2.6ex}}         
\newcommand\Bstrut{\rule[-1.3ex]{0pt}{0pt}}   
\newcommand\Bstrutfrac{\rule[-0.7ex]{0pt}{0pt}}   
\newcommand\Tstrutfrac{\rule{0pt}{2.1ex}}         

\newcommand{\pc}{\zeta_c(t)}
\newcommand{\pd}{\zeta_d(t)}
\newcommand{\pe}{\zeta_e(t)}

\makeatletter
\newcommand\footnoteref[1]{\protected@xdef\@thefnmark{\ref{#1}}\@footnotemark}
\makeatother

\newcommand{\shorteq}{%
  \settowidth{\@tempdima}{-}
  \resizebox{\@tempdima}{\height}{=}%
}

\usepackage{amssymb,fge}

\usepackage{amsmath, amssymb}
\usepackage{subcaption}
\usepackage[utf8]{inputenc}
\usepackage[T1]{fontenc}
\usepackage[linesnumbered,ruled,vlined,algo2e]{algorithm2e}
\usepackage{array}
\usepackage{graphicx}
\usepackage{amsfonts}
\usepackage{soul}
\usepackage{hhline}
\usepackage{multirow, makecell}
\usepackage{float}
\usepackage{booktabs}

\usepackage{amsthm}
\usepackage{color}
\usepackage{transparent}
\usepackage{url}
\usepackage{footmisc}
\usepackage{setspace}
\usepackage{mathtools}

\DeclarePairedDelimiter\abs{\lvert}{\rvert}

\DeclareMathOperator*{\argmax}{arg\,max}
\DeclareMathOperator*{\argmin}{arg\,min}
\theoremstyle{plain}

\newtheorem{theorem}{Theorem}[section]

\newtheorem{definition}{Definition}[section]
\newtheorem{conc}{Conclusion}[]

\newtheorem{problem}{Problem}[section]

\DeclareMathOperator{\EX}{\mathbb{E}}
\title{StylePredict: Machine Theory of Mind for Human Driver Behavior From Trajectories}

\author[]{Rohan Chandra$^1$}
\author[]{Aniket Bera} 
\author[]{Dinesh Manocha}

\affil[]{University of Maryland, College Park}

\leadauthor{Chandra}  

\correspondingauthor{\textsuperscript{1}To whom correspondence should be addressed. E-mail: rchandr1@umd.edu}


\begin{abstract}
Studies have shown that autonomous vehicles (AVs) behave conservatively in a traffic environment composed of human drivers and do not adapt to local conditions and socio-cultural norms. It is known that socially aware AVs can be designed if there exist a mechanism to understand the behaviors of human drivers. We present a notion of Machine Theory of Mind (M-ToM) to infer the behaviors of human drivers by observing the trajectory of their vehicles. Our M-ToM approach, called \textit{StylePredict}, is based on trajectory analysis of vehicles, which has been investigated in robotics and computer vision. StylePredict mimics human ToM to infer driver behaviors, or styles, using a computational mapping between the extracted trajectory of a vehicle in traffic and the driver behaviors using graph-theoretic techniques, including spectral analysis and centrality functions. We use StylePredict to analyze driver behavior in different cultures in the USA, China, India, and Singapore, based on traffic density, heterogeneity, and conformity to traffic rules and observe an inverse correlation between longitudinal (overspeeding) and lateral (overtaking, lane-changes) driving styles. 

\end{abstract}

\begin{document}

\maketitle
\ifthenelse{\boolean{shortarticle}}{\ifthenelse{\boolean{singlecolumn}}{\abscontentformatted}{\abscontent}}{}

\section{Introduction}
\label{}

\dropcap{T}here is considerable and growing interest in developing Theory of Mind~\cite{theoryofmind, cuzzolin2020knowing} in autonomous driving. From studies in traffic psychology~\cite{quartz-article, sa1,sa2,sa3,schwarting2019social}, Theory of Mind (ToM) in autonomous driving corresponds to interpreting driving behavior, or ``styles'', and adjusting to the local driving conditions. The local driving patterns of a particular region vary widely according to the traffic density, the type of vehicles, observance of traffic rules, and communication between drivers based on gestures or honking. Examples of ToM in driving can be observed in the following common case scenarios:

\begin{itemize}
    \item Merging or lane-changing: Suppose driver A wants to merge or switch over to another lane containing driver B. Driver A needs to make an intuitive split-second decision on whether or not driver B will allow the lane-change, or deny the lane-change by accelerating. If B offers the lane-change, the ideal reaction of A is to switch as soon as possible to avoid confusing B. If B denies the lane-change, then A should ideally keep driving and try again.
    
    \item Turning at intersections: At a four-way intersection without traffic lights,  multiple vehicles  arrive at the intersection at the same time. Assume, additionally, that the standard rules of intersections do not hold (right-hand-move-first). A driver, using ToM, would predict the behaviors of the other vehicles and would allow aggressive drivers to pass first.
\end{itemize}

\noindent AVs currently do not posses ToM and therefore exhibit conservative and shortsighted behavior, that frustrates other human drivers~\cite{dirtyTesla-gamma}. For example in~\cite{dirtyTesla-gamma}, a Tesla driver is observed to be executing a lane-change maneuver. The Tesla AutoPilot slows down to wait for an excessively large gap in the target lane thereby blocking the traffic behind it in the current lane. This behavior of Tesla Autopilot frustrates the other human drivers. 

Theory of Mind~\cite{theoryofmind} is a concept in cognitive and social neuroscience and psychology that refers to humans' ability to understand the intentions, desires, and behaviors of other people, without explicit communication. From the neuroscientific perspective, ToM has been explored in applications involving both human and animal subjects primarily to study the relationship between ToM and mental illnesses such as autism and schizophrenia~\cite{frith1994autism}, and demonstrate the existence of ToM in children and aged humans~\cite{wellman1992child}. ToM for human drivers, on the other hand, has not been formally studied, despite a general understanding and acceptance of the need for the former~\cite{quartz-article}.

As human drivers, we possess Theory of Mind, which allows us to infer the behaviors of other human drivers~\cite{quartz-article, cuzzolin2020knowing}, simply by observing the motion of their vehicles. For example, if a vehicle is weaving through traffic, other drivers typically infer that the driver of the weaving vehicle is aggressive, and accordingly keep their distance from that vehicle. ToM thus establishes a \textit{``psychological mapping''} between the trajectory of a vehicle and the behavior of its driver. Problems related to trajectory analysis, for example trajectory prediction~\cite{chandra2019traphic, chandra2019robusttp, chandra2020forecasting} and tracking~\cite{chandra2019densepeds, chandra2019roadtrack}, have been studied mostly in robotics and computer vision, whereas driver behavior modeling has mostly been restricted to traffic psychology and the social sciences~\cite{taubman2004multidimensional, gulian1989dimensions, french1993decision, deffenbacher1994development, ishibashi2007indices,ernestref2,ernestref3,ernestref4,ernestref8,ernestref9,ernestref10,ernestref11,ernestref12,ernestref12,ernestref13,ernestref14,ernestref15,ernestref16}. So far, there is relatively less work to connect the ideas from these disparate fields of research. For AVs to replicate ToM in the same manner as a human driver, new areas of research must bridge the gap between robotics, computer vision, and the social sciences and develop a computational ToM model that maps trajectories to human driver behavior.

Computational models for ToM are also referred to as ``Machine ToM''~\cite{rabinowitz2018machine}, abbreviated as M-ToM. The main strategy of current M-ToM methods in autonomous driving is to learn reward functions for human behavior using inverse reinforcement learning (IRL)~\cite{rabinowitz2018machine, jara2019theory, schwarting2019social, gt2}. M-ToM models that employ IRL, however, have certain major limitations. IRL, as a data-driven approach, requires large amounts of training data, and the learned reward functions are unrealistically tailored towards scenarios only observed in the training data~\cite{rabinowitz2018machine, gt2}. For instance, the approach proposed in~\cite{rabinowitz2018machine} requires $32$ million data samples for optimum performance. Alternatively, Bayesian approaches to IRL are sensitive to noise in the trajectory data~\cite{schwarting2019social, gt2}. Consequently, current M-ToM methods are restricted to simple and sparse traffic conditions.

In light of these limitations, perhaps the main challenge to M-ToM in autonomous driving is to universally model human driver behavior in all kinds of traffic environments. The difficulty stems from the general observation that different traffic environments lead to different driver behaviors because of varying traffic density and heterogeneity. For example  higher traffic density (e.g. India \& China) logically implies smaller spaces in which vehicles can navigate, which inadvertently affects longitudinal driving styles. \textit{Longitudinal} driving styles include any style that is executed along the axis of the road such as overspeeding and underspeeding~\cite{ernestref14, dillen2020keep}. The heterogeneity of traffic, on the other hand, affects \textit{lateral} driving styles, which are executed perpendicular to the axis of the road such as sudden lane-changing, overtaking, and weaving~\cite{dillen2020keep, sagberg2015review}. Two-wheelers and three-wheelers are known to be more likely to perform lateral driving styles~\cite{india-lane}. However, the exact relationship between the nature of the traffic environment and its effect on driver behavior has not been made clear. To model driver behavior in any given region, M-ToM techniques must take into account how driver behaviors depend on the traffic density and heterogeneity of a given region.

\subsection{Main Contributions}

\begin{enumerate}
    \item We present a graph-theoretic Machine Theory of Mind (M-ToM), called \textit{StylePredict}, that computes a \textit{computational mapping} between vehicle trajectories and driver behavior, mimicing the Theory of Mind that humans use to intuitively infer driver behavior by simply observing the vehicle trajectories. 
    
    \item We use graph-theoretic techniques including vertex centrality functions~\cite{rodrigues2019network} and spectral analysis to measure the likelihood and intensity of the driving styles (overspeeding, overtaking, sudden lane-changes, etc). Based on these measures, we assign global behaviors that identify whether a particular driver is aggressive or conservative.
    
    \item We use StylePredict to analyze driver behaviors in different regions of the world and identify a relationship between the nature of the traffic environment and its effect on driver behavior. In particular, we observe an inverse correlation between longitudinal and lateral driving styles of a region. Drivers in regions with \textit{higher} traffic density and heterogeneity (India/China) are more likely to maneuver \textit{laterally} whereas drivers in regions with \textit{lower} traffic density and heterogeneity (U.S.A/Singapore) are more likely to perform \textit{longitudinal} styles.

\end{enumerate}

The main advantages of StylePredict over prior work are the following:

\begin{itemize}
    \item StylePredict is robust and directly operates on raw trajectory data obtained from commodity location sensors such as GPS and stereo-cameras, without any filtering or pre-processing steps. We analyze the robustness property in Section~\ref{sec: approach}\ref{subsec: noise_invariance}.
    
    \item StylePredict generalizes to traffic observed in different geographies and cultures since the traffic information of any region or culture can be reduced to a common graph-based representation (Section~\ref{subsec: DGG}). In this work, we study the behaviors of drivers in Pittsburgh (U.S.A)~\cite{Argoverse}, New Delhi (India)~\cite{chandra2019traphic}, Beijing (China)~\cite{wang2019apolloscape}, and Singapore (private dataset). 
    
    \item StylePredict can be automatically integrated in any realtime autonomous driving system without the need for manual adjustments or parameter-tuning.
\end{itemize}

We demonstrate StylePredict using a state-of-the-art traffic simulator~\cite{leurent2019social} and evaluate its accuracy on real-world traffic videos using the Time Deviation Error (TDE)~\cite{cmetric}. Our results show that StylePredict can model different driver behaviors with near-human accuracy with an average TDE of less than $1$ second.

\begin{table}
\centering
\caption{A list showing the taxonomy of various aggressive and conservative behaviors. In this work, we focus on modeling the longitudinal and Lateral specific styles (except ``tailgating'' and `` responding to pressure''). NA denotes ``Not Applicable''}
\resizebox{\columnwidth}{!}{
\begin{tabular}{lcc}
\toprule
Global Behavior & Specific Style/Indicator & Nature   \\
\midrule
\multirow{7}{*}{Aggressive} & Overspeeding  & Longitudinal\\
&  Tailgating & Longitudinal\\
& Overtake as much as possible & Lateral\\
 & Weaving & Lateral \\
& Sudden Lane-Changes & Lateral\\
& Inappropriate use of horn & NA\\
 &  Flashing lights at vehicle in front & NA\\
\midrule
\multirow{3}{*}{Conservative} & Uniform speed or under-speeding & Longitudinal \\
 & Conforms to single lane & Lateral \\
 & Responding to pressure from other drivers & NA \\

\bottomrule
\end{tabular}
}
\label{tab: list_of_behaviors}
\end{table}

\subsection{Interpretation of Driver Behavior in Social Science }
\label{subsec: driver_behavior}

Many studies have attempted to define driver behavior for traffic-agents~\cite{def1,def2,def3,def4,def5,def6}. 
However, due to its abstract nature, driver behavior does not lend itself to a formal definition. For example in~\cite{def1,def2}, driver behavior is described as driving habits that are established over a period of time, while Ishibashi et al.~\cite{def3} define driver behavior as ``an attitude, orientation and way of thinking for driving''. To resolve these inconsistencies, Sagberg et al.~\cite{sagberg2015review} extract and summarize the common elements from these definitions and propose a unified definition for driver behavior. We incorporate this definition in our driver behavior model.

\begin{definition}
\textbf{(Sagberg et al.~\cite{sagberg2015review}} Driver behavior refers to the high-level \textit{``global behavior''}, such as aggressive or conservative driving. Each global behavior can be expressed as a combination of one or more underlying \textit{``specific styles''}. For example, an aggressive driver (global behavior) may frequently overspeed or overtake (specific styles).
\label{def: behavior}
\end{definition}

\noindent The main benefit of Sagberg's definition is that it allows for a formal taxonomy for driver behavior classification. Specific indicators can be classified as either \textit{longitudinal} styles (along the axis of the road) or \textit{lateral} (perpendicular to the axis of the road). Any combination of these indicators can define a global behavior such as aggressive or conservative driving. We list several longitudinal and lateral specific styles corresponding to the aggressive and conservative global behaviors in Table~\ref{tab: list_of_behaviors}.

\subsection{Problem Setup}

We can formally characterize driver behavior by mathematically modeling the underlying specific styles.

\begin{problem}
In a traffic video with $N$ vehicles during any time-period $\Delta t$, given the trajectories of all vehicles, our objective is to \ul{mathematically model} the \textit{specific styles} for all drivers during $\Delta t$.
\label{problem: prob}
\end{problem}

\noindent In Section~\ref{sec: approach}, we will elucidate what it means to ``mathematically model'' a specific style. In Section~\ref{subsec: DGG}, we construct the ``traffic-graph'' data structure used by our approach. Then we introduce the ideas of vertex centrality in Section~\ref{sec: centrality} followed by a presentation of our main approach in Section~\ref{sec: approach}. We describe the experiments and results in Section~\ref{sec: experiments_and_results}. We analyze the traffic in different cultures and list some observations in Section~\ref{sec: culture_analysis} and conclude the paper in Section~\ref{sec: conclusion}.

\begin{figure*}
\centering
\includegraphics[width=\linewidth]{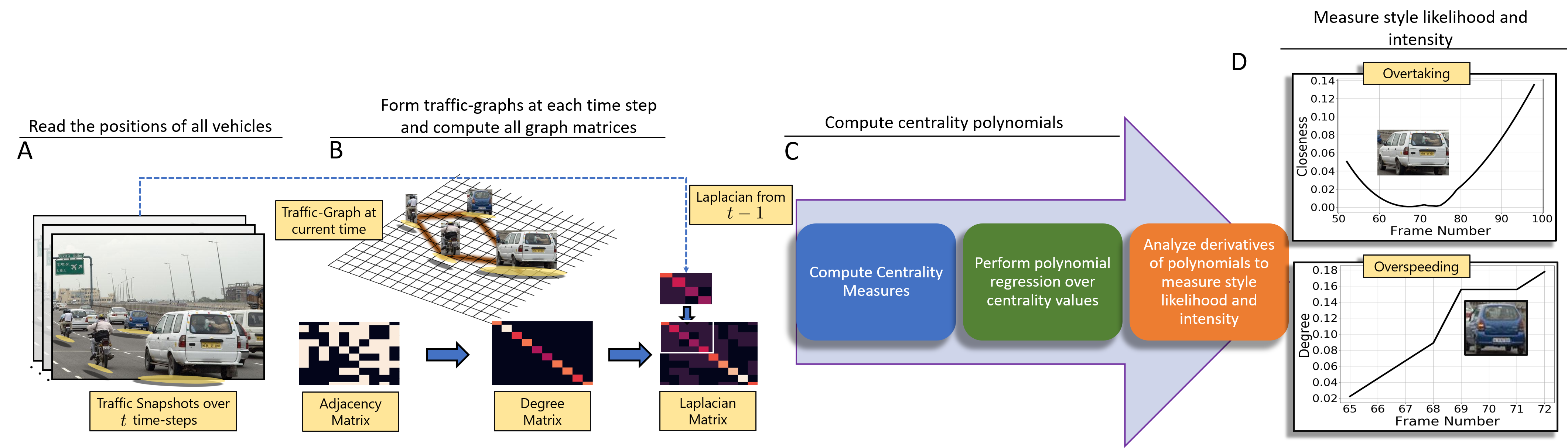}
\caption{\textbf{Overview:} The autonomous vehicle reads the positions of all vehicles in realtime. The positions and corresponding spatial distances between vehicles are represented through a traffic-graph $\mathcal{G}_t$ (Section~\ref{subsec: DGG}). We use the centrality functions defined in Section~\ref{sec: centrality} to model the specific driving style corresponding to the global behaviors as outlined in Table~\ref{tab: behaviors_centrality}.}
\label{fig: overview}
\end{figure*}

\section{Traffic-Graph Setup}

\label{subsec: DGG}

The behavior of a driver depends on his or her interactions with nearby drivers. StylePredict models the relative interactions between drivers by representing traffic through weighted undirected graphs, called ``traffic-graphs''. In this section, we describe the construction of these graph representations. If we assume that the trajectories of all the vehicles in the video are extracted using state-of-the-art localization methods~\cite{bresson2017simultaneous} and are provided to our algorithm as an input, then the traffic-graph, $\mc{G}_t$, at each time-step $t$ can be defined as follows,

\begin{definition}
A ``traffic-graph'', $\mc{G}_t$, is a dynamic, undirected, and weighted graph with a set of vertices $\mathcal{V}(t)$ and a set of edges $\mc{E}(t) \subseteq \mc{V}(t) \times \mc{V}(t)$ as functions of time defined in the 2-D Euclidean metric space with metric function $f(x,y) = \Vts{x-y}^2$. Two vertices $v_i, v_j \in \mc{V}$ are connected if and only if $f(v_i,v_j) < \mu$ for some constant $\mu$.
\end{definition}

We represent traffic at each time instance $t$ with $N$ road-agents using a traffic-graph $\mc{G}_t$. Each vertex in the graph $\mc{G}_t$ is represented by the vehicle position in the global coordinate frame, \textit{i.e.} $v_i \gets [x_i,y_i]^\top \in \mathbb{R}^2$. The spatial distance between two vehicles is assigned as the cost of the edge connecting the two vehicles.

In computational graph theory, every graph $\mc{G}_t$ can be equivalently represented by an adjacency matrix, $A_t$. For a particular traffic-graph, $\mc{G}_t$, the adjacency matrix, $A \in \mathbb{R}^{N \times N}$ is given by,

\begin{equation}
A(i,j)=
     \begin{cases}
      f(v_i,v_j) & \text{if $f(v_i,v_j) < \mu,i \neq j$ },\\
      0 &\text{otherwise.}
     \end{cases}
     \label{eq: similarity_function}
\end{equation}
\noindent where $\mu$ is a distance threshold parameter. Adjacency matrices allow linear vector operations to be performed on graph structures, which are useful for analyzing individual vertices. For example, from Equation~\ref{eq: similarity_function}, it can be observed that each non-zero entry in the $j^\textrm{th}$ column corresponding to the $i^\textrm{th}$ row of the adjacency matrix stores the relative distance between the $i^\textrm{th}$ and $j^\textrm{th}$ vehicles. 



However, considering the traffic-graph and its corresponding adjacency matrix only at a current time-step $t$ is not useful in describing the behavior of a driver. The behavior of a driver also depends on their actions from previous time-steps~\cite{gt2}. To accommodate this notion, we enforce the adjacency matrix at the current time-step to be correlated with the adjacency matrices for all previous time-steps. If the adjacency matrix at a time instance $t$ is denoted as $A_t$, then, the adjacency matrix for the next time-step, $A_{t+1}$, is obtained by the following update,

\begin{equation}
A_{t+1} =
\left[
\begin{array}{c|c}
A_{t} \Bstrut & 0 \Bstrut\\
\hline
0 \Tstrut & 1
\end{array}
\right] + \delta\delta^\top,
\label{eq: A_update}
\end{equation}

\noindent where $\delta \in \bb{R}^{(t+1) \times 2} $ is a sparse update matrix. The presence of a non-zero value in the $j^\textrm{th}$ row of $\delta$ indicates that the $j^\textrm{th}$ road-agent has formed an edge connection with a new vehicle, that has been added to the current traffic-graph. The update rule in Equation~\ref{eq: A_update} ensures that a vehicle adds edge connections to new vehicles while retaining edge connections with previously seen vehicles. A candidate vehicle is categorized as ``new'' with respect to an ego-vehicle if there does not exist any prior edge connection between the vehicles \textit{and} the speed of the ego-vehicle is greater than the candidate vehicle. If an edge connection already exists between an ego-vehicle and the candidate vehicle, then the candidate vehicle is said to have been ``observed'' or ``seen''. The size of $A_t$ is fixed for all time-instants, $t$, and is initialized as a zero matrix of size $N$x$N$, where $N$ is the max number of agents. $A_t$ is updated in-place with time and is reset to a zero matrix once the number of vehicles crosses some fixed $N$.

\section{Centrality Measures}
\label{sec: centrality}
\begin{table*}[t]
\centering
\caption{Definition and categorization of driving behaviors~\cite{sagberg2015review}. We measure the likelihood and intensity of specific styles by analyzing the first-and second-order derivatives of the centrality polynomials.}
\centering
\resizebox{\linewidth}{!}{
\begin{tabular}{lcccc}

\toprule

Global Behaviors & Specific Styles & Centrality & Style Likelihood Estimate & Style Intensity Estimate\\
\midrule
 \multirow{5}{*}{Aggressive}& Overspeeding & Degree ($\zeta_d$) & Magnitude of $1^\textrm{st}$ Derivative & Magnitude of $2^\textrm{nd}$ Derivative\\
 & Overtaking / Sudden Lane-Change & Closeness ($\zeta_c$) & Magnitude of $1^\textrm{st}$ Derivative  &   Magnitude of $2^\textrm{nd}$ Derivative\\
  & Weaving  & Closeness ($\zeta_c$) & Local Extreme Points & $\varepsilon$-sharpness of Local Extreme Points \\
 \midrule

\multirow{2}{*}{Conservative}& Driving Slowly or uniformly  & Degree ($\zeta_d$)& Magnitude of $1^\textrm{st}$ Derivative& Magnitude of $2^\textrm{nd}$ Derivative\\
 & No Lane-change  & Closeness ($\zeta_c$)& Magnitude of $1^\textrm{st}$ Derivative& Magnitude of $2^\textrm{nd}$ Derivative\\

\bottomrule
\end{tabular}
}
\label{tab: behaviors_centrality}
\end{table*}

 In graph theory and network analysis, centrality measures are real-valued functions $\zeta:\mathcal{V}\longrightarrow \mathbb{R}$, where $\mathcal{V}$ denotes the set of vertices and $\mathbb{R}$ denotes a scalar real number that identifies key vertices within a graph network. So far, centrality functions have been restricted to identifying influential personalities in online social media networks~\cite{cen-socialmedia} and key infrastructure nodes in the Internet~\cite{cen-internet}, to rank web-pages in search engines~\cite{cen-pagerank}, and to discover the origin of epidemics~\cite{cen-epidemic}. In this work, we show that centrality functions can measure the likelihood and intensity of different driver styles such as overspeeding, overtaking, sudden lane-changes, and weaving.
 
There are several types of centrality functions. The ones that are of particular importance to us are the degree and closeness centrality~\cite{rodrigues2019network}. Each function measures a different property of a vertex. Typically, the choice of selecting a centrality function depends on the current application at hand. For instance, in the examples mentioned earlier, the eigenvector centrality is used in search engines while the degree centrality is appropriate for measuring the popularity of an individual in social networks. Likewise, we found that different centrality functions measure the likelihood and intensity of specific driving styles such as overspeeding, overtaking, sudden lane-changes, and weaving. Here, we give the basic definitions of the key centrality functions used in our approach.



\begin{definition}
\textbf{Closeness Centrality:} In a connected traffic-graph at time $t$ with adjacency matrix $A_t$, let $\mc{D}_t(v_i,v_j)$ denote the minimum total edge cost to travel from vertex $i$ to vertex $j$, then the discrete closeness centrality measure for the $i^\textrm{th}$ vehicle at time $t$ is defined as,
\begin{equation}
    \zeta^i_c[t] = \frac{N-1}{\sum_{v_j\in \mathcal{V}(t)\setminus \{v_i\}} \mc{D}_t(v_i,v_j)},
    \label{eq: closeness}
\end{equation}
\end{definition}

The closeness centrality for a given vertex computes the reciprocal of the sum of the edge values of the shortest paths between the given vertex and all other vertices in the connected traffic-graph. By definition, the higher the closeness centrality value, the more centrally the vertex is located in the graph. 

\begin{definition}
\textbf{Degree Centrality:} In a connected traffic-graph at time $t$ with adjacency matrix $A_t$, let $\mc{N}_i(t) = \{ v_j \in \mc{V}(t), \ A_t(i,j) \neq 0, \nu_j \leq \nu_i\}$ denote the set of vehicles in the neighborhood of the $i^\textrm{th}$ vehicle with radius $\mu$, then the discrete degree centrality function of the $i^\textrm{th}$ vehicle at time $t$ is defined as,
\begin{equation}
    \begin{aligned}
    \zeta^i_d[t] = \bigl | \{ v_j \in \mc{N}_i(t) \} \bigr | + \zeta^i_d[t-1] &\\
    \textrm{such that} \ (v_i,v_j) \not\in \mc{E(\tau)}, \tau = 0, \ldots, t-1&
    \end{aligned}
    \label{eq: degree}
\end{equation}
where $\vts{\cdot}$ denotes the cardinality of a set and $\nu_i, \nu_j$ denote the velocities of the $i^\textrm{th}$ and $j^\textrm{th}$ vehicles, respectively.
\end{definition}

The degree centrality at any time-step $t$ computes the cumulative number of edges between the given vertex and connected vertices in the traffic-graph over the past $t$ seconds, given that the velocity of the given vertex is higher than that of the connected vertices. 

\begin{algorithm}[t]
    \SetKwInOut{Input}{Input}
    \SetKwInOut{Output}{Output}
\SetKwComment{Comment}{$\triangleright$\ }{}
\SetAlgoLined
\Input{$u = v_i \gets [x_i,y_i]^\top \ \forall v_i \in \mathcal{V}(t)$}
\Output{$\textrm{SLE}(t), \textrm{SIE}(t)$}
$t=0$\\
\For {each $v \in \mathcal{V}(t)$}{
\While{$t \leq T$}{ 
// Compute Centrality //\\
$     \zeta^i_c[t] = \frac{N-1}{\Tstrutfrac \sum_{v_j\in \mathcal{V}(t)\setminus \{v_i\}} \mc{D}_t(v_i,v_j)}$ \\ 
$    \zeta^i_d[t] = \bigl | \{ v_j \in \mc{N}_i(t) \} \bigr | + \zeta^i_d[t-1], (v_i,v_j) \not\in \mc{E(\tau)}, \tau = 0, \ldots, t-1$ \\
$t \gets t+1$ \\
// Perform Polynomial Regression w.r.t Time //\\
$\beta = \argmin_\beta \Vts{\zeta - M\beta}$\\
$\zeta(t) = \beta_0 + \beta_1 t + \beta_2t^2$\\
// Compute Likelihood and Intensity //\\
\For{$k=0,1,2$}{
$\textrm{SLE}_k(t) = \abs*{\frac{\partial \zeta(t)}{\partial t}}$\\
$\textrm{SIE}_k(t) = \abs*{\frac{\partial^2 \zeta(t)}{\partial t^2}}$
}
}
}
\caption{Our approach outputs the Style Likelihood Estimate (SLE) and Style Intensity Estimate (SIE) for a vehicle, $u$, in a given time-period $\Delta t$. }
\label{alg: main}
\end{algorithm}

\section{StylePredict: Mapping Trajectories to Behavior}
\label{sec: approach}

Here, we present the main algorithm, called \textit{StylePredict}, for solving Problem~\ref{problem: prob}. StylePredict maps vehicle trajectories to specific styles by computing the likelihood and intensity of the latter using the definitions of the centrality functions (Section~\ref{sec: centrality}). The specific styles are then used to assign global behaviors~\cite{sagberg2015review} according to Table~\ref{tab: behaviors_centrality}.

\paragraph{StylePredict Algorithm}

\begin{enumerate}
    \item Obtain the positions of all vehicles using localization sensors deployed on the autonomous vehicle and form traffic-graphs at each time-step (Section~\ref{subsec: DGG}).
    
    
    \item Compute the closeness and degree centrality function values for each vehicle at every time-step.
    
    \item Perform polynomial regression to generate uni-variate polynomials of the centralities as a function of time. 
    
    \item Measure likelihood and intensity of a specific style for each vehicle by analyzing the first- and second-order derivatives of their centrality polynomials.
    
\end{enumerate}

\noindent We depict the overall approach in Figure~\ref{fig: overview} and outline the pseudo-code in Algorithm~\ref{alg: main}. We begin by using the construction described in Section~\ref{subsec: DGG} to form the traffic-graphs for each frame and use the definitions in Section~\ref{sec: centrality} to compute the discrete-valued centrality measures (lines $5-7$). Since centrality measures are discrete functions, we perform polynomial regression using regularized Ordinary Least Squares (OLS) solvers to transform the three centrality functions into continuous polynomials, $\pc$, $\pd$, and $\pe$, as a function of time (lines $10-11$). We describe polynomial regression in detail in the following subsections. We compute the likelihood and intensity (lines $12-13$) of specific styles by analyzing the first- and second-order derivatives of $\pc$, $\pd$, and $\pe$, respectively (this step is discussed in further detail in Section~\ref{subsec: SLE_SIE}).



\subsection{Polynomial Regression}
\label{subsec: polynomial_regression}
 
In order to study the behavior of the centrality functions with respect to how they change with time, it is useful to convert the discrete-valued $\zeta[t]$ into continuous-valued polynomials $\zeta(t)$. This allows us to calculate the first- and second-order derivatives of the centrality functions necessary to measure the likelihood and intensity of driver behavior, explained in Section~\ref{sec: centrality}\ref{subsec: SLE_SIE}.

A quadratic\footnote{A polynomial with degree $2$.} centrality polynomial can be expressed as $\zeta(t) = \beta_0 + \beta_1 t + \beta_2t^2$, as a function of time. Here, $\beta = [\beta_0 \ \beta_1 \ \beta_2]^\top$ are the polynomial coefficients. These
 coefficients can be computed by solving the following ordinary least squares (OLS) problem:


\begin{equation}
    \begin{split}
        M \beta &= \zeta^i  \\
        (M^\top M) \beta &= M^\top \zeta^i  \\
        \beta &= \inv{(M^\top M)}M^\top \zeta^i
    \end{split}
    \label{eq: noiseless_OLS}
\end{equation}

\noindent Here, $M \in \mathbb{R}^{T\times(d+1)}$ is the Vandermonde matrix and is given by, 

\[
M = \begin{bmatrix}
    1 & t_1 & t^2_1 & \dots  & t^d_1 \\
    1 & t_2 & t^2_2 & \dots  & t^d_2 \\
    \vdots & \vdots & \vdots & \ddots & \vdots \\
    1 & t_T & t^2_T & \dots  & t^d_T
\end{bmatrix}
\]

\noindent where $d = 2 \ll T$ is the degree of the resulting centrality polynomial.

\begin{figure}[t]
\centering
\begin{subfigure}[h]{\linewidth}
  \includegraphics[width=\linewidth]{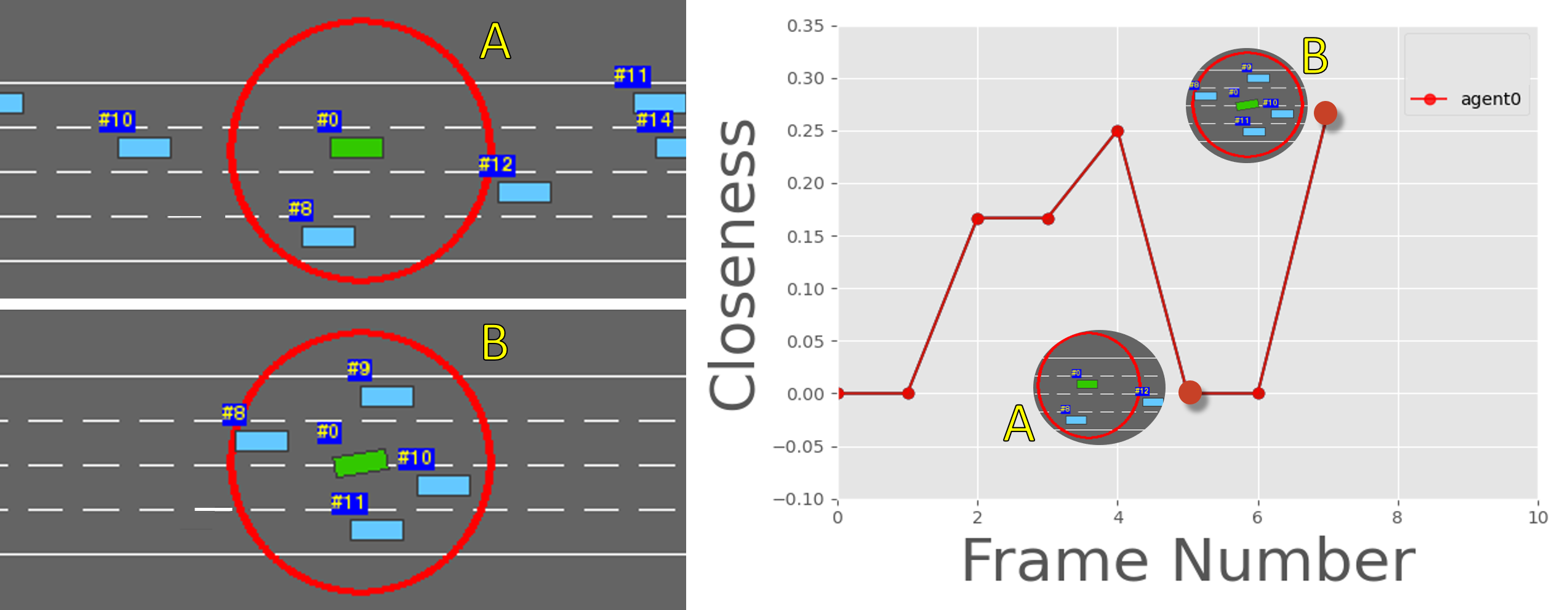}
\caption{\textbf{Sudden Lane Change/Weaving: }The closeness centrality for the green agent \#0 decreases away from the center (scenario A) and increases towards the center (scenario B).}
\label{fig: closeness_demo}
\end{subfigure}
\begin{subfigure}[h]{\linewidth}
  \includegraphics[width=\linewidth]{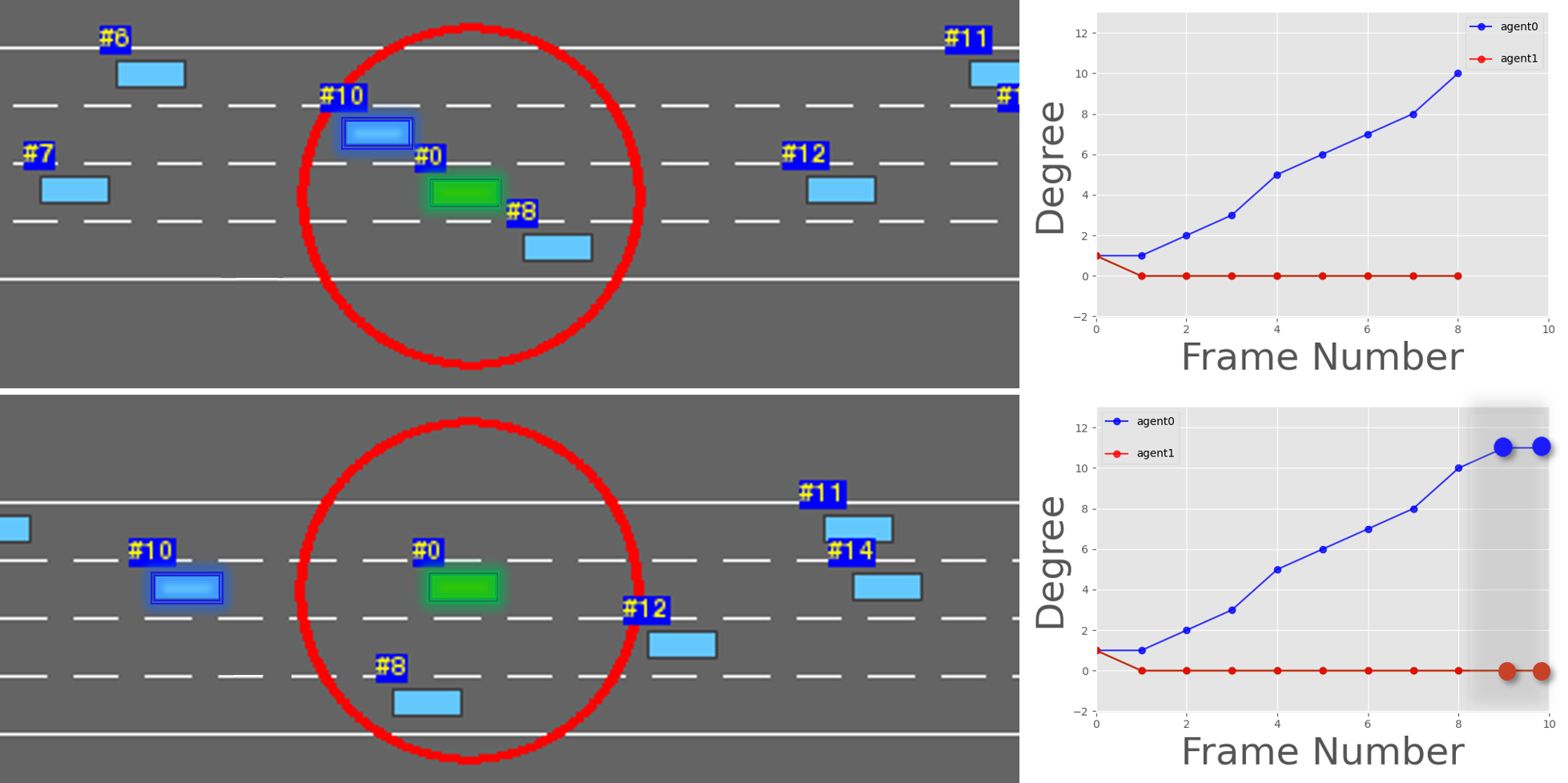}
\caption{\textbf{Overspeeding: }The degree centrality monotonically increases. The rate of increase indicates the scale of aggressiveness, with higher rates of change indicating more aggressively overspeeding drivers. The blue plot corresponds to agent \#0, while the red plot corresponds to agent \#10. Note that agent \#10 is being overtaken by agents that appeared earlier (\#0 and \#8) due to it's conservative speed. Consequently, it's degree centrality plot is constant at $0$.}
\label{fig: degree_demo}
\end{subfigure}
\caption{Simulation results of various driving styles (overspeeding, weaving, lane changing) with their corresponding centrality values. The red circle represents the radius for the agent \#0.}
\label{fig: simulator_figs}
\end{figure}

\begin{figure*}[t]
\centering
  \begin{subfigure}[h]{\textwidth}
    \includegraphics[width=\textwidth]{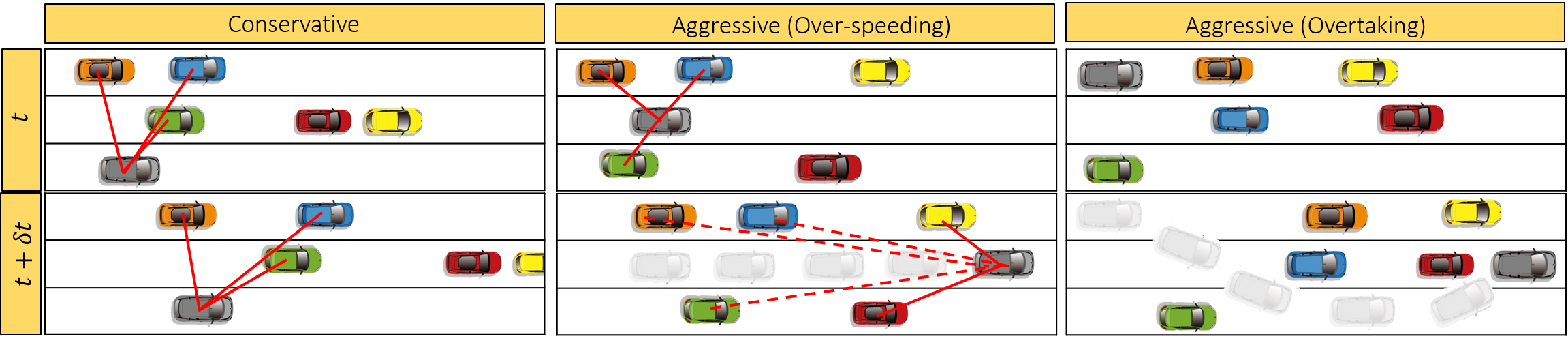}
    \caption{In all three scenarios, the ego-vehicle is a gray vehicle marked with a blue outline. (\textit{left}) A conservative vehicle, \textit{(middle)} an overspeeding vehicle in the same lane, and \textit{(right)} a weaving and overtaking vehicle. }
    \label{fig: part_a}
  \end{subfigure}
  \begin{subfigure}[h]{0.328\textwidth}
    \includegraphics[width=\textwidth]{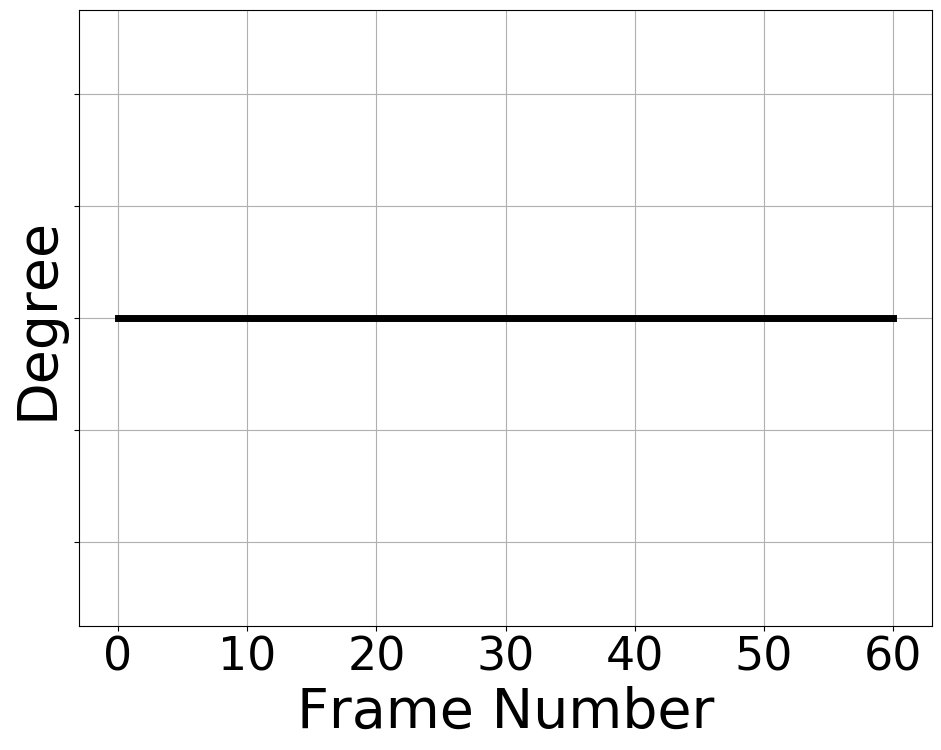}
    \caption{Constant degree centrality function for conservative vehicle.}
    \label{fig: part_b}
  \end{subfigure}
    \begin{subfigure}[h]{0.332\textwidth}
    \includegraphics[width=\textwidth]{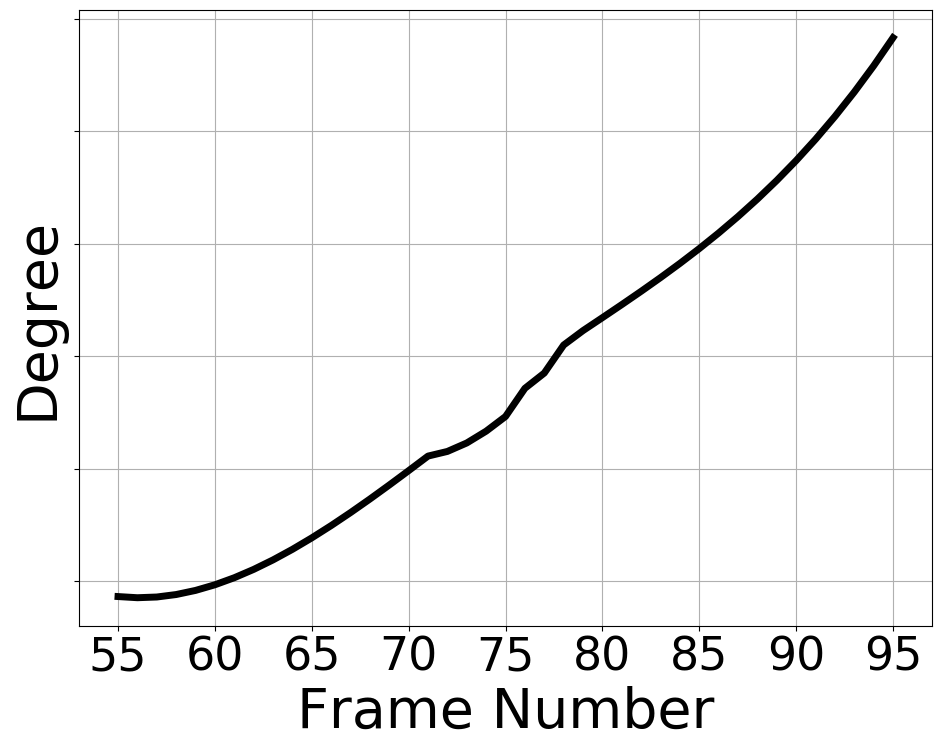}
    \caption{Monotonically increasing centrality function for overspeeding vehicle.}
    \label{fig: part_c}
  \end{subfigure}
  \begin{subfigure}[h]{0.324\textwidth}
    \includegraphics[width=\textwidth]{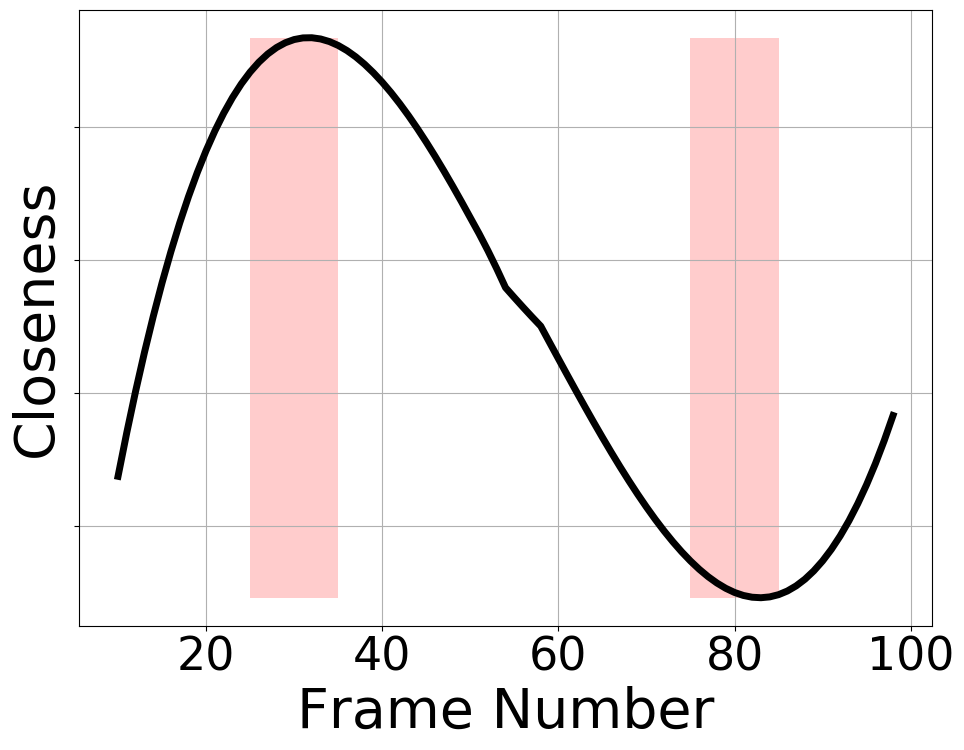}
    \caption{Extreme points for closeness centrality function for weaving vehicle.}
    \label{fig: part_d}
  \end{subfigure}
  \caption{\textbf{Measuring the Likelihood of Specific Styles:} We measure (degree and closeness centrality) the likelihood of the specific style of the ego-vehicle (grey with a blue outline) by computing the magnitude of the derivative of the centrality functions as well as the functions' extreme points. In Figure~\ref{fig: part_b}, the derivative of the degree centrality function is $0$ because the ego-vehicle does not observe any additional new neighbors (See Section~\ref{subsec: SLE_SIE}) so the degree centrality is a constant function; therefore the vehicle is conservative. In Figure~\ref{fig: part_c}, the vehicle overspeeds, and consequently, the rate of observing new neighbors is high, which is reflected in the magnitude of the derivative of the degree centrality being positive. Finally, in Figure~\ref{fig: part_d}, the ego-vehicle demonstrates overtaking/sudden lane-changes and weaves through traffic. This is reflected in the magnitude of the slope and the location of extreme points, respectively, of the closeness centrality function.}
  \label{fig: maneuver_behaviors}
\end{figure*}

\subsection{Robustness to Noise}
\label{subsec: noise_invariance}

In the formulation above, our algorithm assumes perfect sensor measurements of the global coordinates of all vehicles. However, in real-world systems, even state-of-the-art methods for vehicle localization incur some error in measurements. We consider the case when the raw sensor data is corrupted by some noise $\epsilon$. Without loss of generality, we prove robustness to noise for the degree centrality and the analysis can be extended to other centrality functions. The discrete-valued centrality vector for the $i^\textrm{th}$ agent is given by $\zeta^i \in \mathbb{R}^{T\times 1}$. So $\zeta^1[2]$ corresponds to the degree centrality value of the $1^\textrm{st}$ agent at $t=2$. 

In the previous section, we show that a noiseless estimator may be obtained by solving an ordinary least squares (OLS) system given by Equation~\ref{eq: noiseless_OLS}.  However, in the presence of noise $\epsilon$, the OLS system described in Equation~\ref{eq: noiseless_OLS} is modified as shown below:

\begin{equation}
    \begin{split}
        &M \tilde\beta = \tilde\zeta^i  \\
        &\tilde\beta = \inv{(M^\top M)}M^\top \tilde\zeta^i
    \end{split}
    \label{eq: noisy_OLS}
\end{equation}

\noindent where $\tilde \zeta^i = \zeta^i + \epsilon$. Then we can prove the following,

\begin{theorem}
$\Vts{\tilde \beta - \beta} =  \bigO{\epsilon}$.
\label{thm: noiseless}
\end{theorem}

\noindent We defer the proof to the supplementary material.

\subsection{Style Likelihood and Intensity Estimates}
\label{subsec: SLE_SIE}
In the previous sections, we used polynomial regression on the centrality functions to compute centrality polynomials. In this section, we analyze and discuss the first and second derivatives of the degree centrality, $\pd$, and closeness centrality, $\pc$, polynomials. Based on this analysis, which may vary for each specific style, we compute the Style Likelihood Estimate (SLE) and Style Intensity Estimate (SIE)~\cite{cmetric}, which are used to measure the probability and the intensity of a specific style.







\paragraph{Overtaking/Sudden Lane-Changes}
Overtaking is the act of one vehicle going past another vehicle, traveling in the same or an adjacent lane, in the same direction. From Definition~\ref{eq: closeness} in Section~\ref{sec: centrality}, the value of the closeness centrality increases as the vehicle moves towards the center and decreases as it moves away from the center. The SLE of overtaking can be computed by measuring the first derivative of the closeness centrality polynomial using $\textrm{SLE}(t) = \abs*{\frac{\partial \pc}{\partial t}}$. The maximum likelihood $\textrm{SLE}_\textrm{max}$ can be computed as $\textrm{SLE}_{\textrm{max}} = \max_{t \in \Delta t}{\textrm{SLE}}(t)$. The SIE of overtaking is computed by simply measuring the second derivative of the closeness centrality using~$ \textrm{SIE}(t) = \abs*{\frac{\partial^2 \pc}{\partial t^2}}$. Sudden lane-changes follow a similar maneuver to overtaking and therefore can be modeled using the same equations used to model overtaking. We visualize the use of closeness centrality to model overtaking in Figure~\ref{fig: closeness_demo}.

\paragraph{Overspeeding}

We use the degree centrality discussed in Section~\ref{sec: centrality} to model overspeeding. As $A_t$ is formed by adding rows and columns to $A_{t-1}$ (See Equation~\ref{eq: A_update}), the degree of the $i^\textrm{th}$ vehicle (denoted as $\theta_i$) is calculated by simply counting the number of non-zero entries in the $i^\textrm{th}$ row of $A_t$. Intuitively, an aggressively overspeeding vehicle will observe new neighbors (increasing degree) at a higher rate than neutral or conservative vehicles. Let the rate of increase of $\theta_i$ be denoted as $\theta_i^{'}$. By definition of the degree centrality and construction of $A_t$, respectively, the degree centrality for an aggressively overspeeding vehicle will monotonically increase. Conversely, the degree centrality for a conservative vehicle driving at a uniform speed or braking often at unconventional spots such as green light intersections will be relatively flat. Figure~\ref{fig: degree_demo} visualizes how the degree centrality can distinguish between an overspeeding vehicle and a vehicle driving at a uniform speed.  Therefore, the likelihood of overspeeding can be measured by computing,
\[\textrm{SLE}(t) = \abs*{\frac{\partial \pd}{\partial t}}\]

\noindent Similar to overtaking, the maximum likelihood estimate is given by $\textrm{SLE}_{\textrm{max}} = \max_{t  \in \Delta t}{\textrm{SLE}}(t)$. Figures~\ref{fig: part_b} and~\ref{fig: part_c} visualize how the degree centrality can distinguish between an overspeeding vehicle and a vehicle driving at a uniform speed. 

\paragraph{Weaving} Weaving is the act of a vehicle shifting its position from a side lane towards the center, and vice-versa~\cite{farah2009passing}. In such a scenario, the closeness centrality function values oscillates between low values on the side lanes and high values towards the center. Mathematically, weaving is more likely to occur near the critical points (points at which the function has a local minimum or maximum) of the closeness centrality polynomial. The critical points $t_c$ belong to the set $\mathcal{T} =\big\{ t_c \big | \frac{\partial \zeta_c(t_c)}{\partial t} = 0 \big\}$. Note that $\mathcal{T}$ also includes time-instances corresponding to the domain of constant functions that characterize conservative behavior. We disregard these points by restricting the set membership of $\mathcal{T}$ to only include those points $t_c$ whose $\varepsilon-$sharpness~\cite{dinh2017sharp} of the closeness centrality is non-zero. The set $\mathcal{T}$ is reformulated as follows,

\begin{equation}
    \begin{aligned}
    \mathcal{T} = \Bigg \{ t_c \bigg | \frac{\partial \zeta_c(t_c)}{\partial t} = 0 \Bigg \} \\
    \textrm{s.t.} \max_{t \in \mathcal{B}_\varepsilon(t_c)} \frac{\partial \pc}{\partial t}  \neq \frac{\partial \zeta_c(t_c)}{\partial t}\\
    \end{aligned}
    \label{eq: weaving}
\end{equation}

\noindent where $\mathcal{B}_\varepsilon (y) \in \mathbb{R}^d$ is the unit ball centered around a point $y$ with radius $\varepsilon$. The SLE of a weaving vehicle is represented by $\vts{\mathcal{T}}$, which represents the number of elements in $\mathcal{T}$. The $\textrm{SIE}(t)$ is computed by measuring the $\varepsilon-$sharpness value of each $t_c \in \mathcal{T}$. Figure~\ref{fig: part_d} visualizes how the degree centrality can distinguish between an overspeeding vehicle and a vehicle driving at a uniform speed. 

\paragraph{Conservative Vehicles}

Conservative vehicles, on the other hand, are not inclined towards aggressive maneuvers such as sudden lane-changes, overspeeding, or weaving. Rather, they tend to conform to a single lane~\cite{ahmed1999modeling} as much as possible, and drive at a uniform speed~\cite{sagberg2015review}, typically at or below the speed limit. The values of the closeness and degree centrality functions in the case of conservative vehicles thus, remain constant. Mathematically, the first derivative of constant polynomials is $0$. The SLE of conservative behavior is therefore observed to be approximately equal to $0$. Additionally, the likelihood that a vehicle drives uniformly in a single lane during time-period $\Delta t$ is higher when,

\[\abs*{\frac{\partial \pc}{\partial t}} \approx 0 \  \textrm{and} \ \max_{t \in \mathcal{B}_\varepsilon(t^{*})} \textrm{SLE}(t) \approx \textrm{SLE}(t_c) .\]

\noindent The intensity of such maneuvers will be low and can be reflected in the lower values for the SIE. 

\begin{algorithm}[t]
    \SetKwInOut{Input}{Input}
    \SetKwInOut{Output}{Output}
\SetKwComment{Comment}{$\triangleright$\ }{}
\SetAlgoLined
\Input{$M$ participants, set of starting frames $S = \{s_1, s_2, \ldots, s_M\}$, set of ending frames $E = \{e_1, e_2, \ldots, e_M\}$ }
\Output{$\EX[T]$ for a video}
$s^* = \min S$\\
$e^* = \max E$\\
Initialize a counter $c_t = 0$ for each frame $t \in [s^*, e^*]$\\
\For {$t \in [s^*, e^*]$}{
\If{$t \in [s_m, e_m]$}{
$c_t \gets c_t + 1$
}
$\mathcal{P}(T=t) = c_t $
}
$\EX[T] = \sum_t tc_t$, $t = s^*, s^*+1, \ldots, e^*$
\caption{Computing $\EX[T]$ for each video in a dataset.}
\label{alg: ex}
\end{algorithm}

\begin{figure*}[t]
\centering
\begin{subfigure}[h]{0.24\linewidth}
    \includegraphics[width=\textwidth]{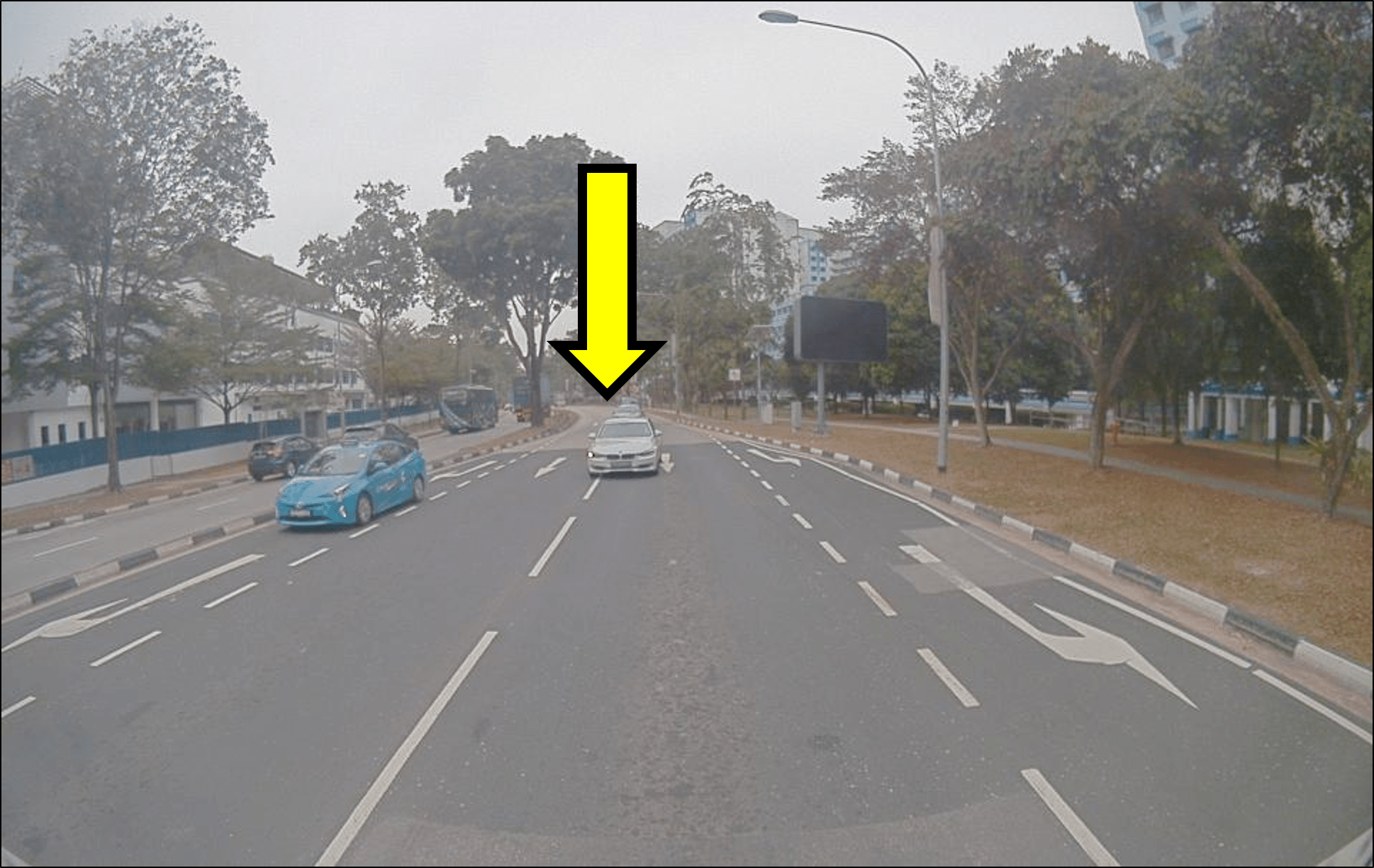}
    \caption{Frame $65$.}
    \label{fig: sg1}
  \end{subfigure}
   \begin{subfigure}[h]{0.24\linewidth}
    \includegraphics[width=\textwidth]{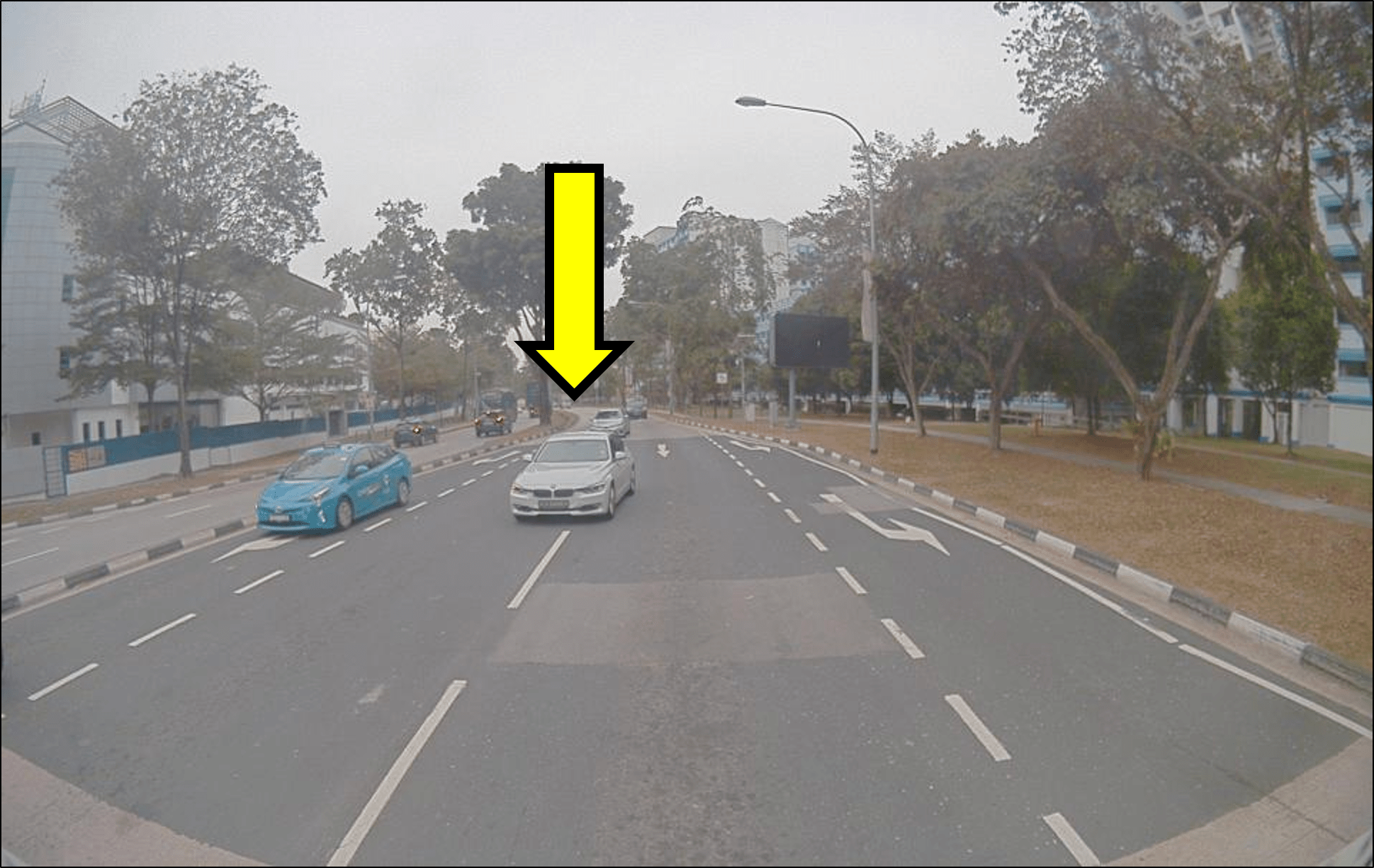}
    \caption{Frame $70$.}
    \label{fig: sg2}
  \end{subfigure}
  \begin{subfigure}[h]{0.24\linewidth}
    \includegraphics[width=\textwidth]{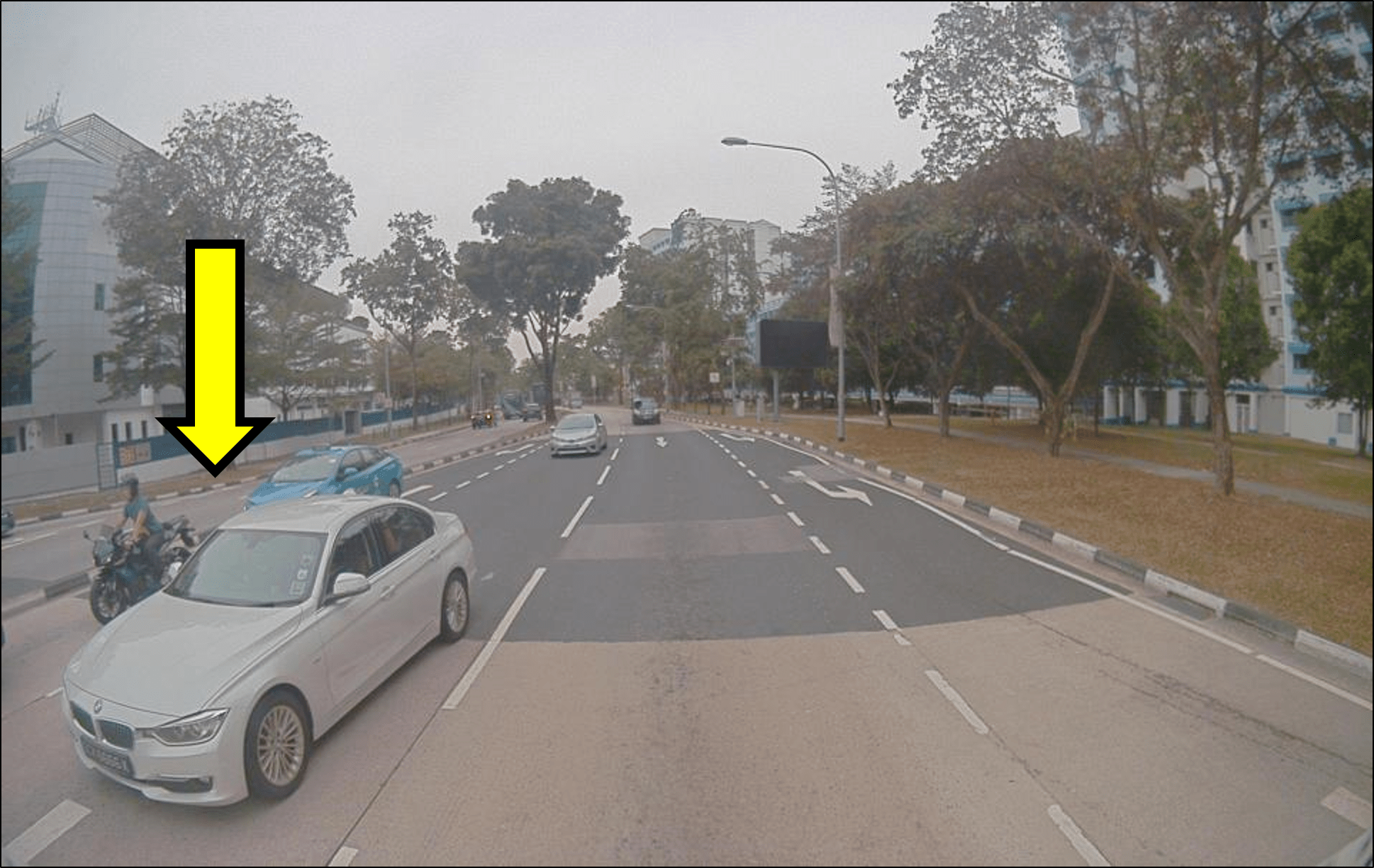}
    \caption{Frame $75$.}
    \label{fig: sg3}
  \end{subfigure}
  \begin{subfigure}[h]{0.24\linewidth}
    \includegraphics[width=\textwidth]{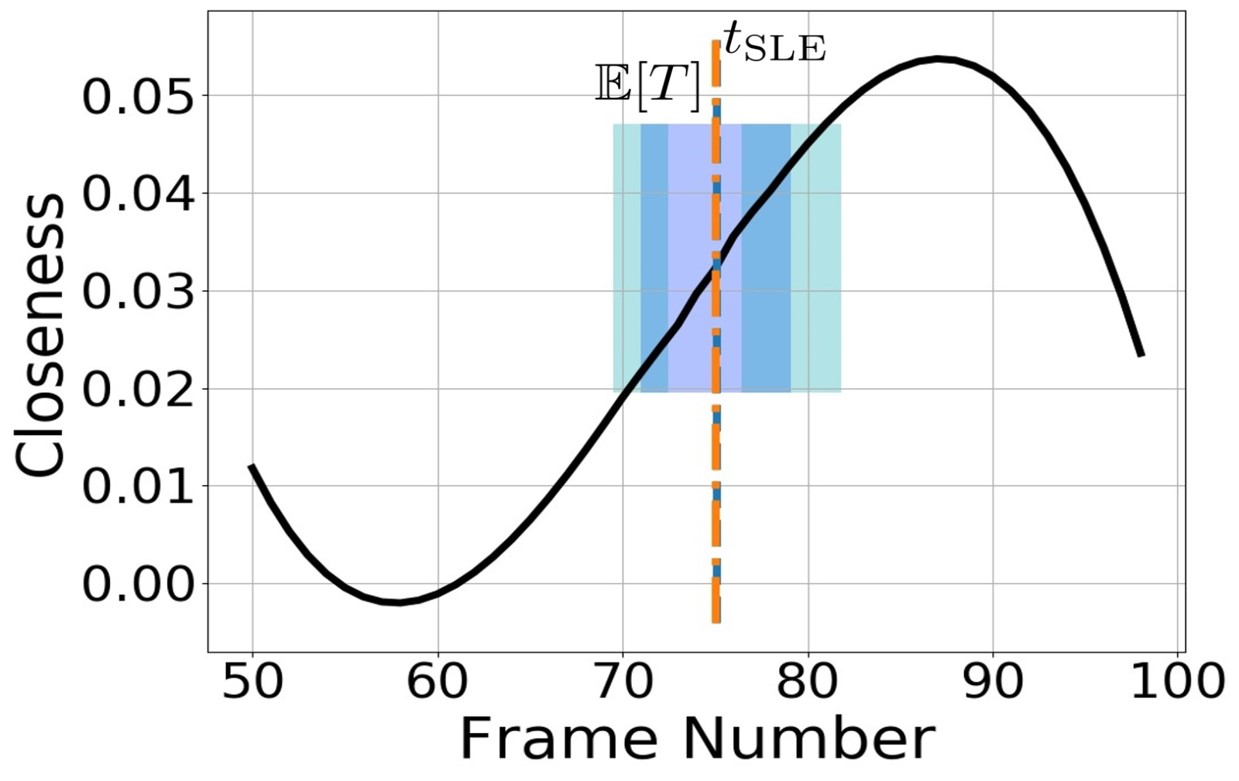}
    \caption{Sudden Lane-Change.}
    \label{fig: sg4}
  \end{subfigure}
  %
  %
    \begin{subfigure}[h]{0.24\linewidth}
    \includegraphics[width=\textwidth]{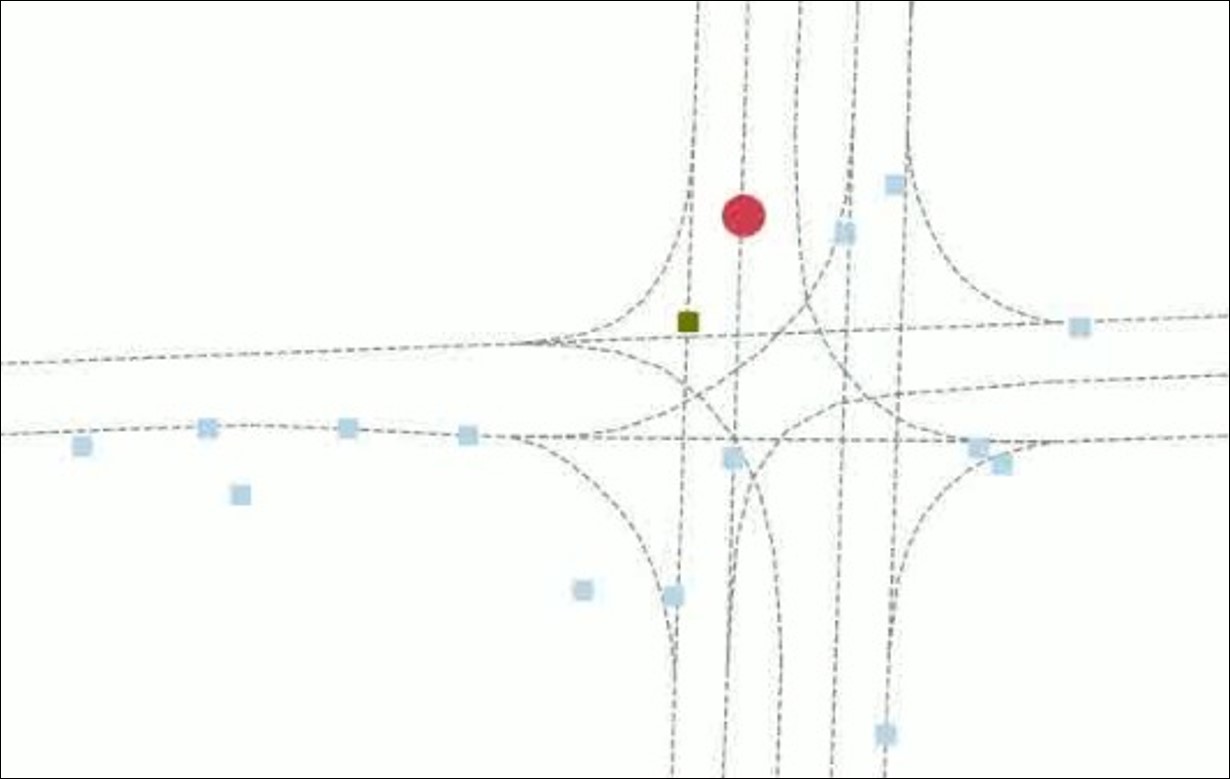}
    \caption{Frame $20$.}
    \label{fig: argo1}
  \end{subfigure}
  \begin{subfigure}[h]{0.24\linewidth}
    \includegraphics[width=\textwidth]{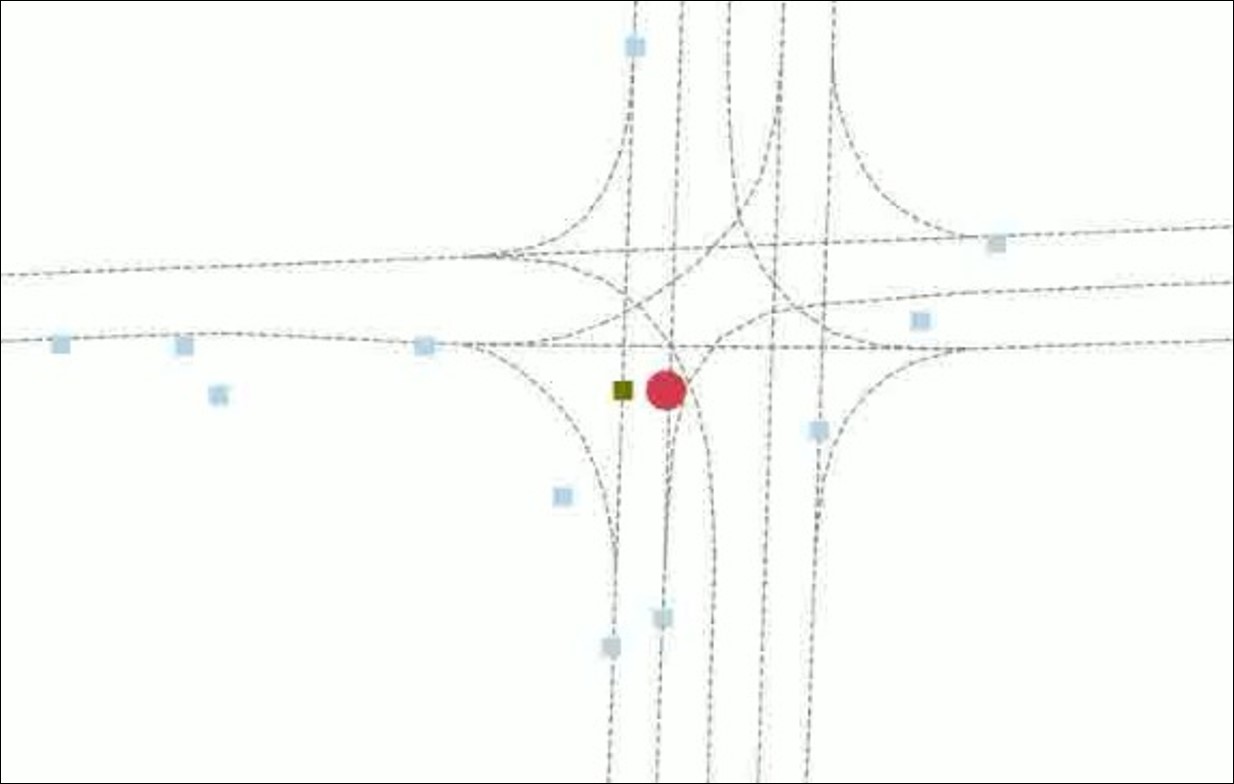}
    \caption{Frame $25$.}
    \label{fig: argo2}
  \end{subfigure}
    \begin{subfigure}[h]{0.24\linewidth}
    \includegraphics[width=\textwidth]{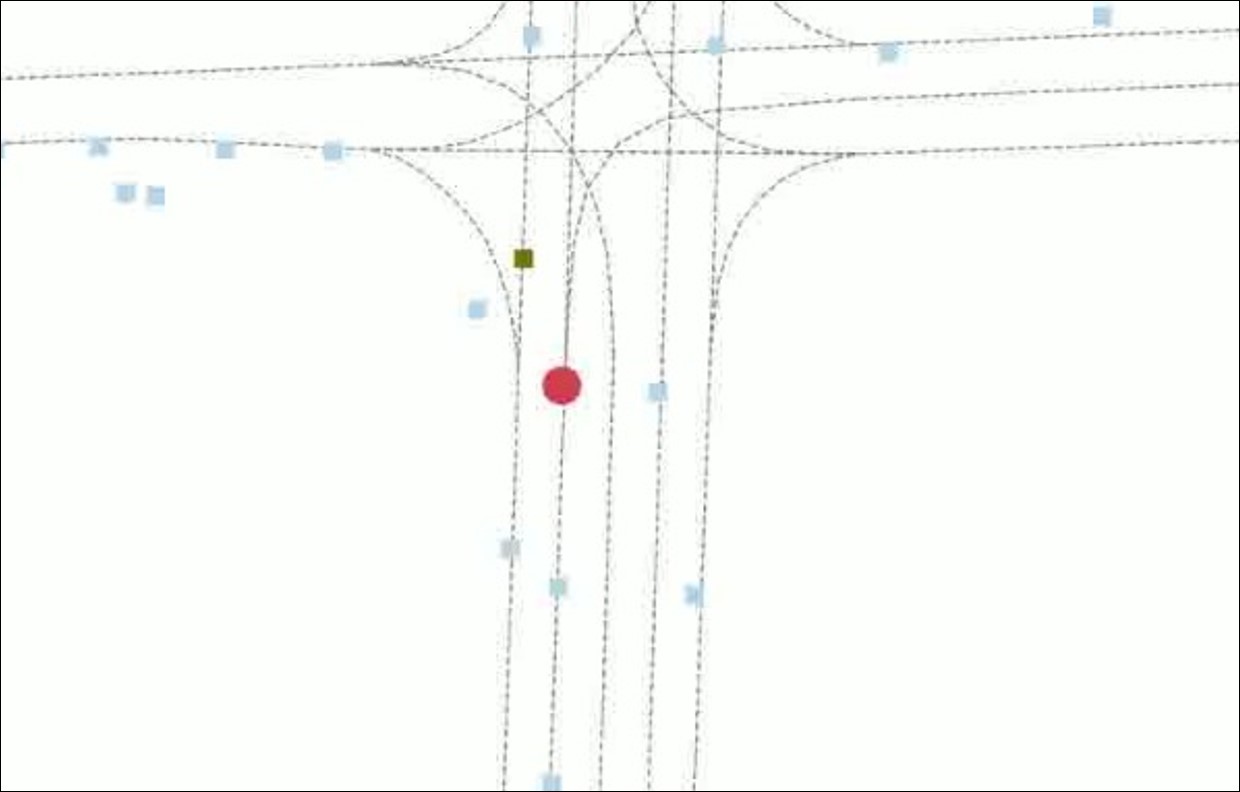}
    \caption{Frame $20$.}
    \label{fig: argo3}
  \end{subfigure}
    \begin{subfigure}[h]{0.24\linewidth}
    \includegraphics[width=\textwidth]{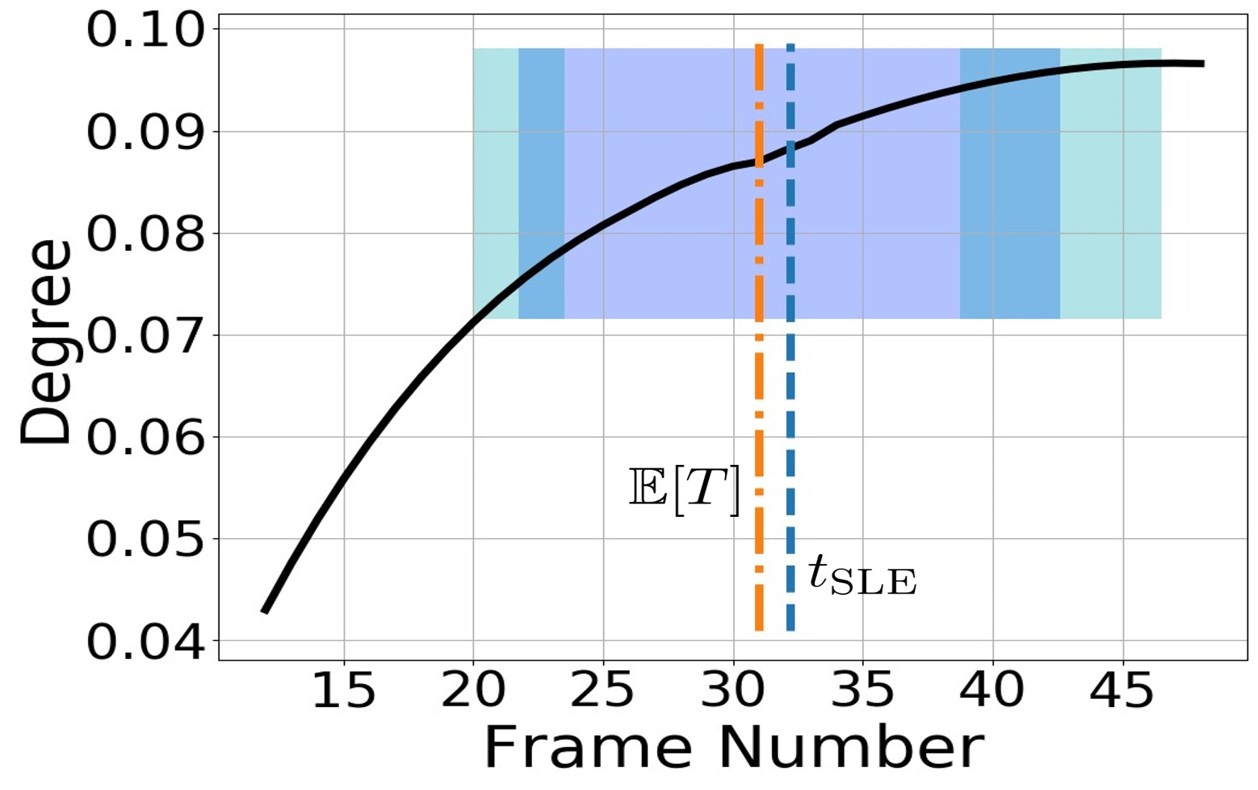}
    \caption{Overspeeding.}
    \label{fig: argo4}
  \end{subfigure}
    \begin{subfigure}[h]{0.24\linewidth}
    \includegraphics[width=\textwidth]{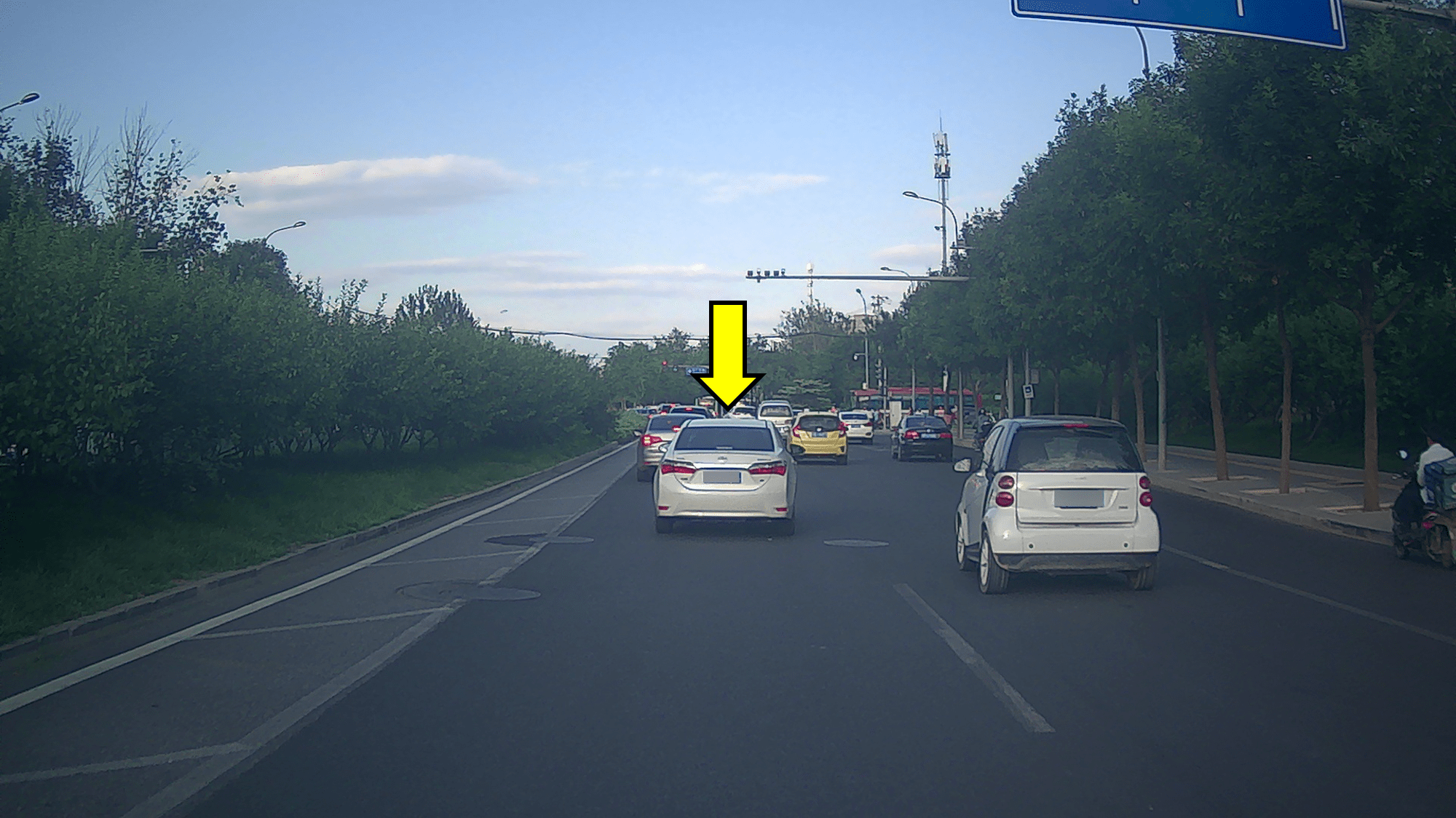}
    \caption{Frame $75$.}
    \label{fig: argo1}
  \end{subfigure}
  \begin{subfigure}[h]{0.24\linewidth}
    \includegraphics[width=\textwidth]{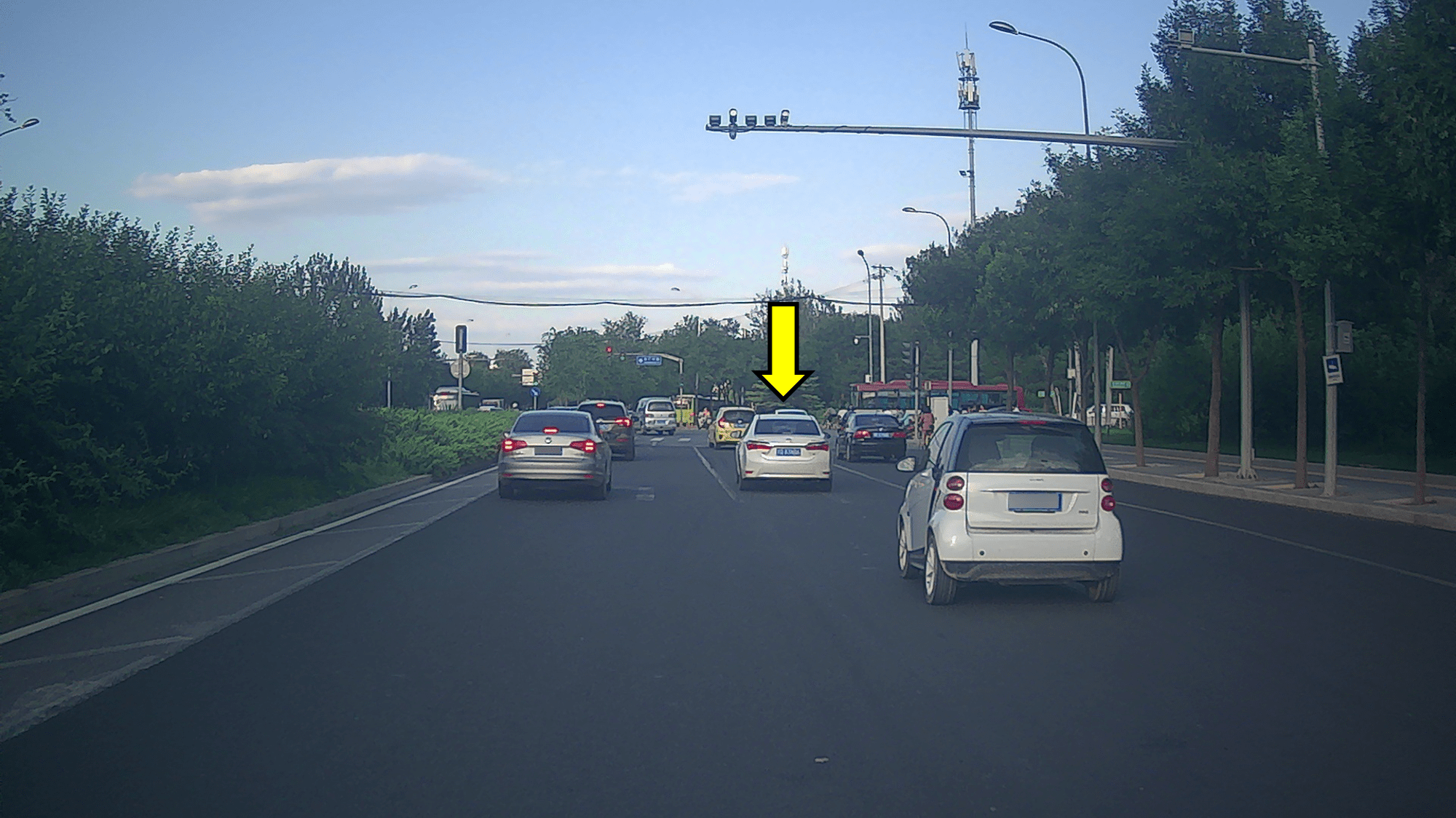}
    \caption{Frame $80$.}
    \label{fig: argo2}
  \end{subfigure}
    \begin{subfigure}[h]{0.24\linewidth}
    \includegraphics[width=\textwidth]{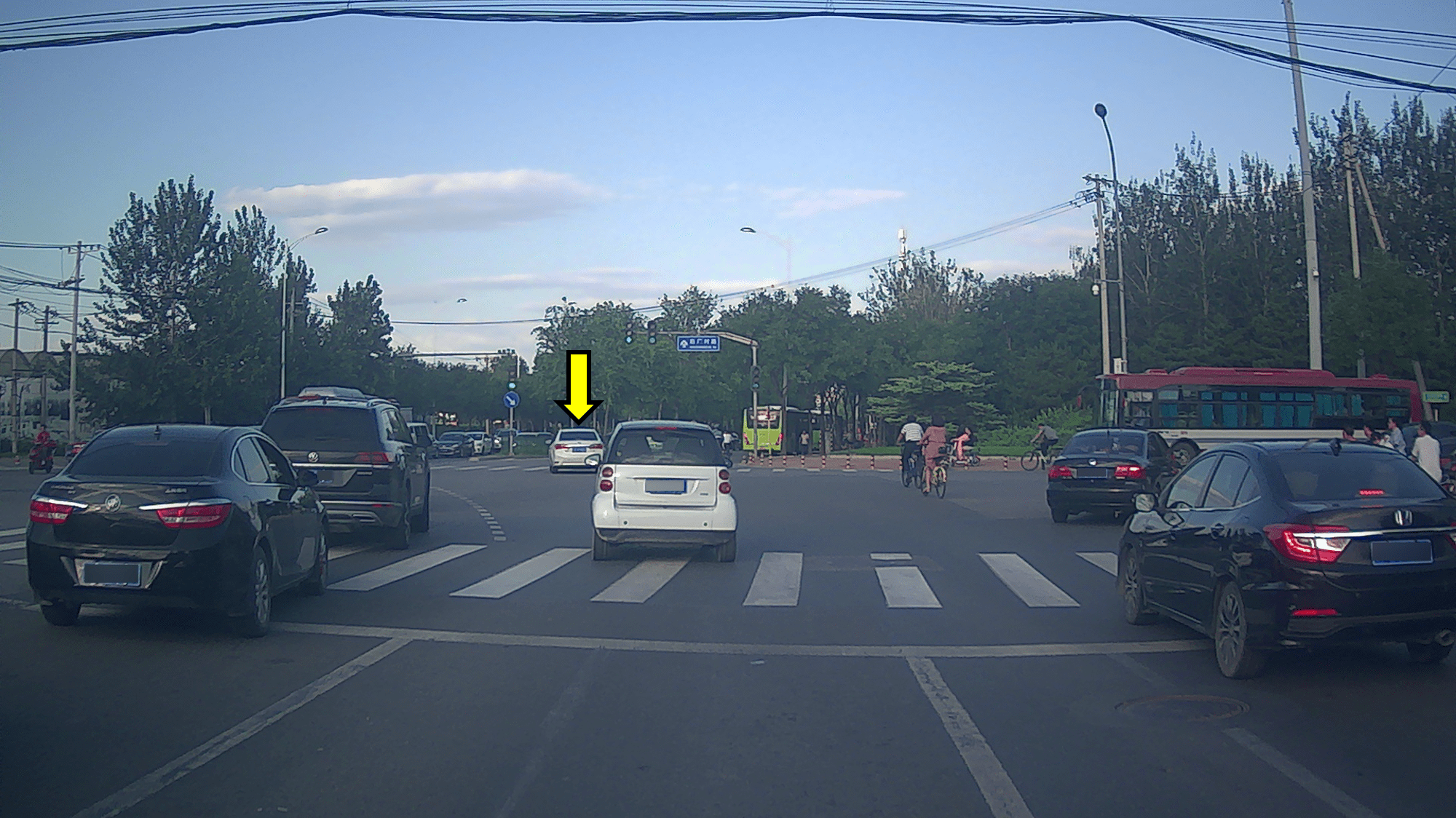}
    \caption{Frame $85$.}
    \label{fig: argo3}
  \end{subfigure}
    \begin{subfigure}[h]{0.24\linewidth}
    \includegraphics[width=\textwidth]{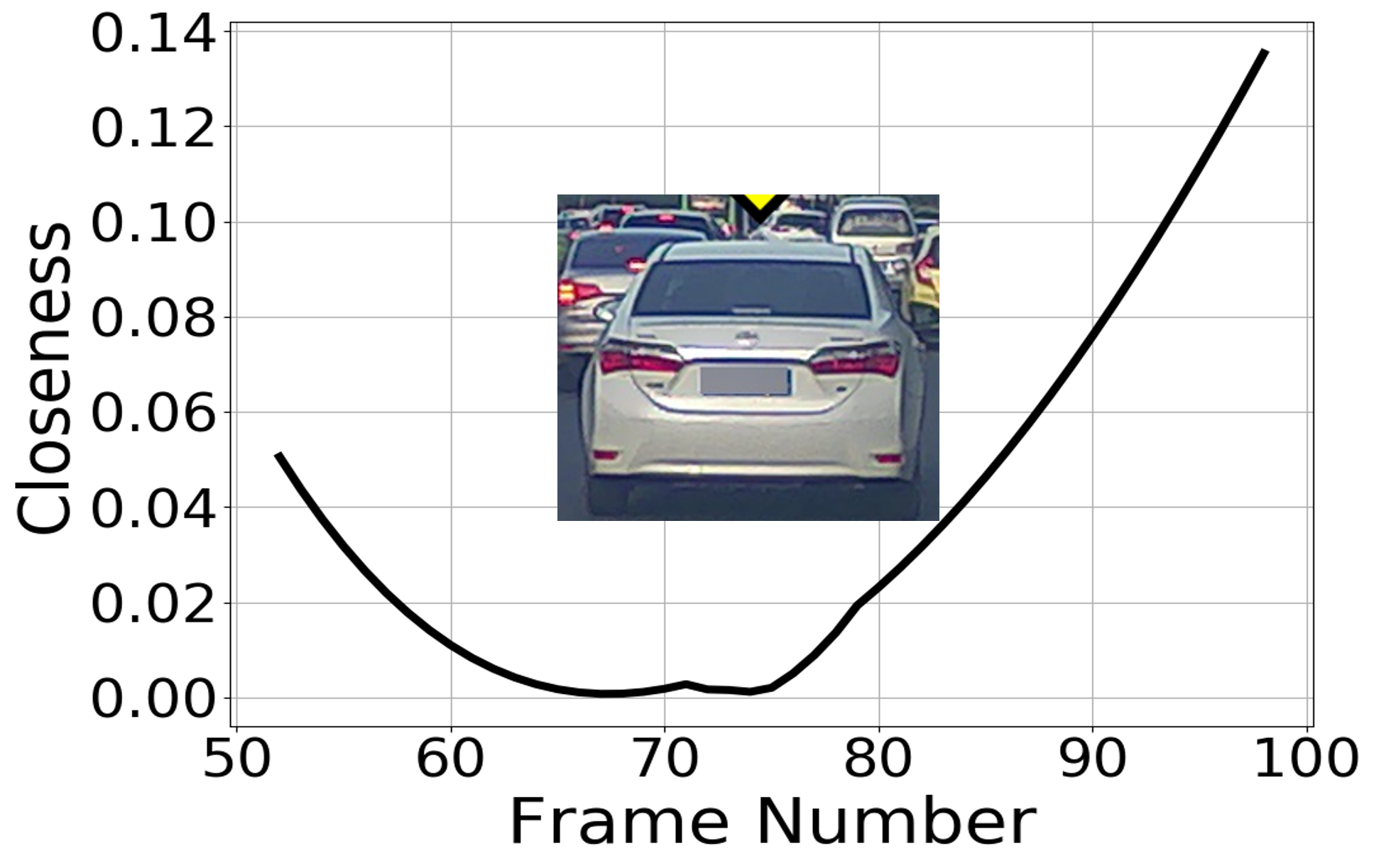}
    \caption{Weaving.}
    \label{fig: argo4}
  \end{subfigure}
    \begin{subfigure}[h]{0.24\linewidth}
    \includegraphics[width=\textwidth]{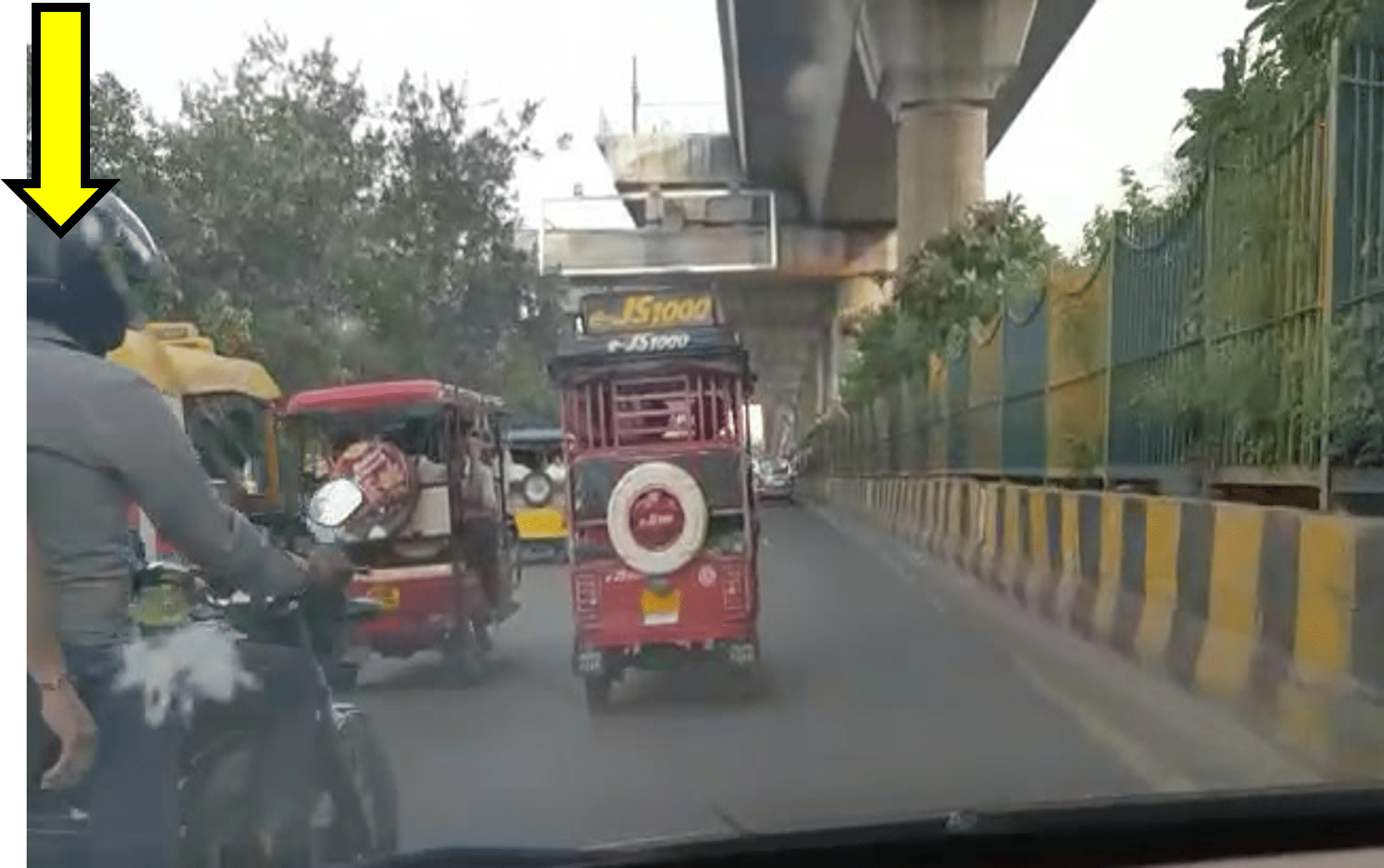}
    \caption{Frame $66$.}
    \label{fig: argo1}
  \end{subfigure}
  \begin{subfigure}[h]{0.24\linewidth}
    \includegraphics[width=\textwidth]{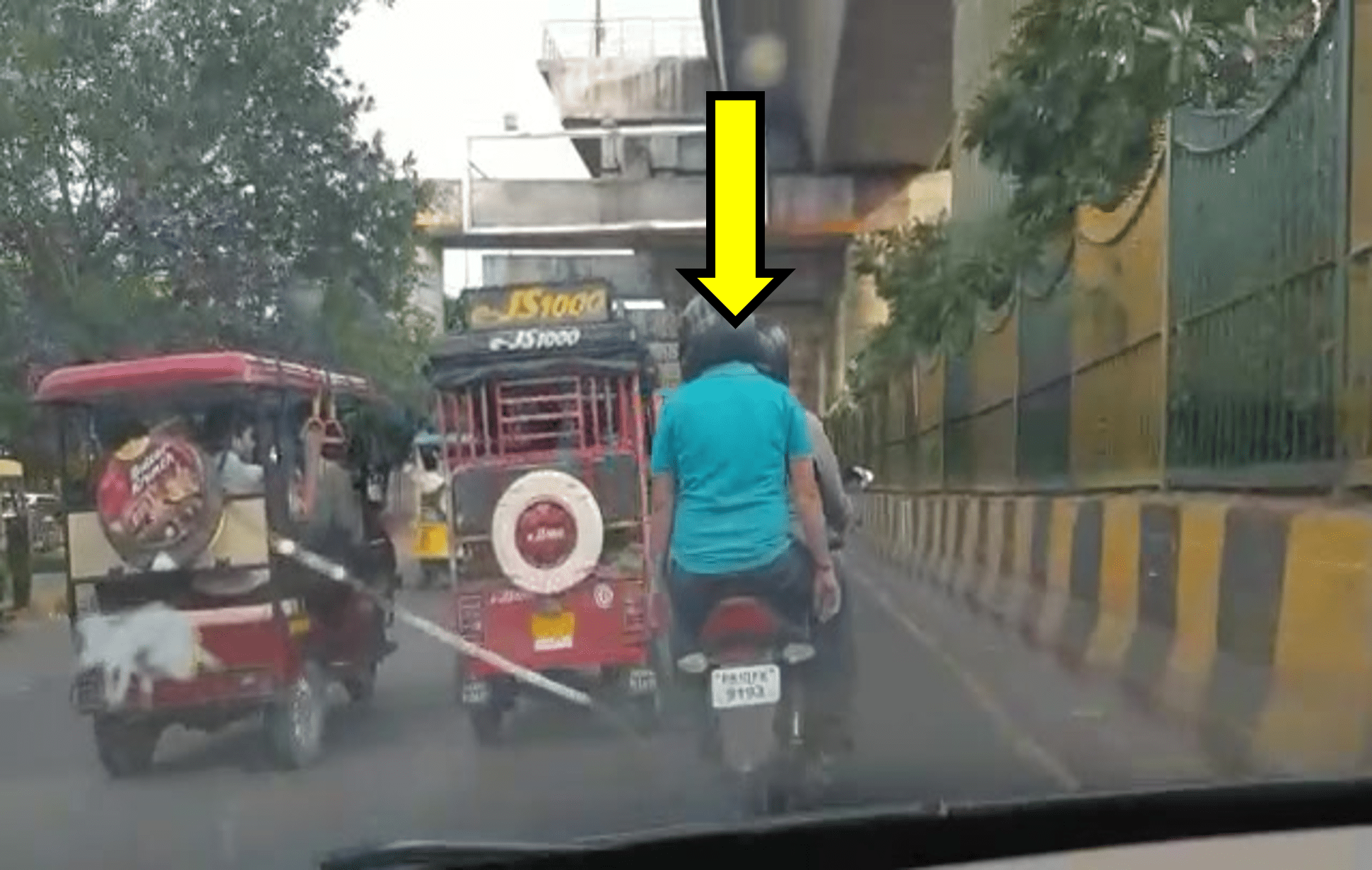}
    \caption{Frame $68$.}
    \label{fig: argo2}
  \end{subfigure}
    \begin{subfigure}[h]{0.24\linewidth}
    \includegraphics[width=\textwidth]{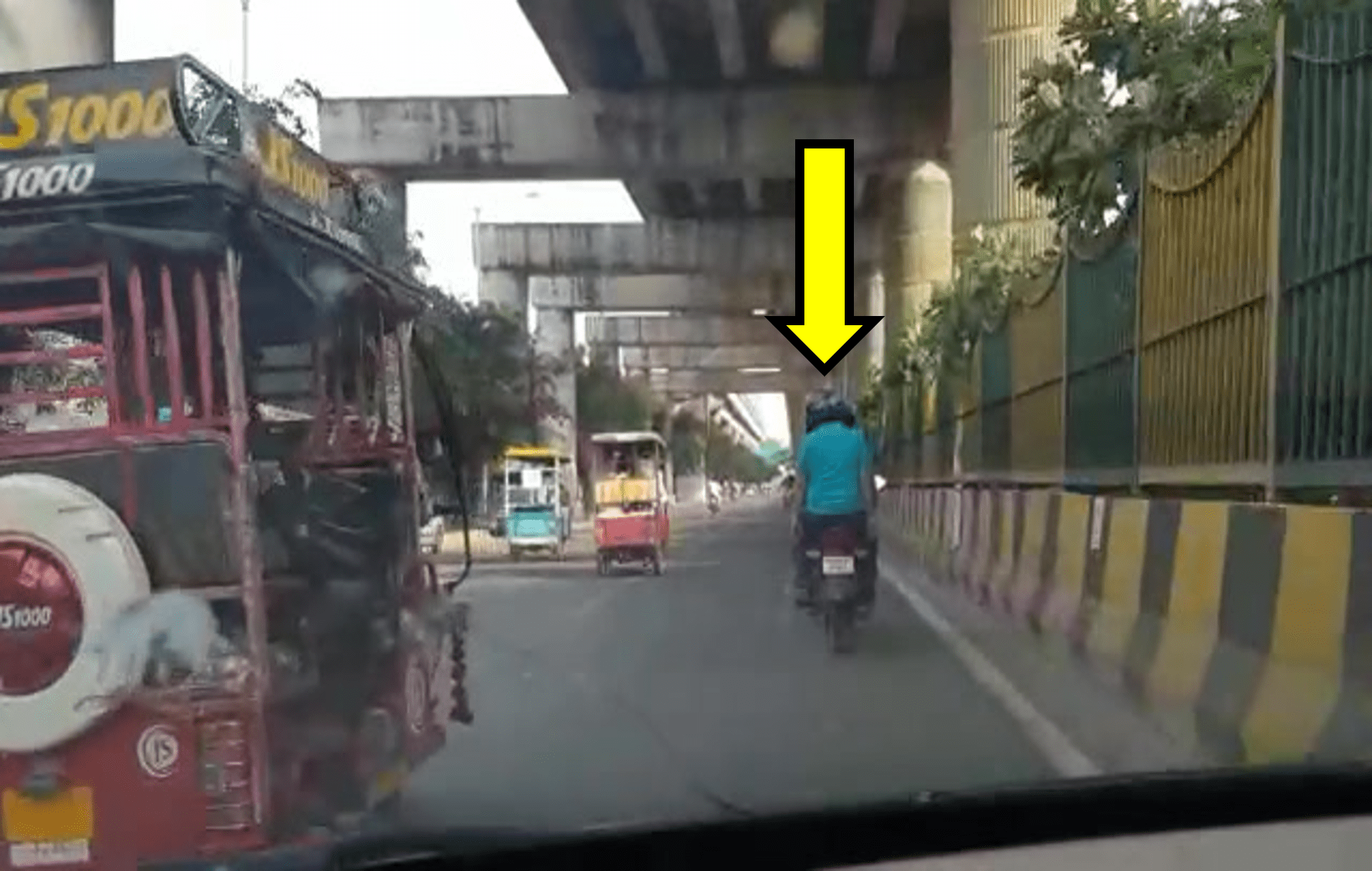}
    \caption{Frame $70$.}
    \label{fig: argo3}
  \end{subfigure}
    \begin{subfigure}[h]{0.24\linewidth}
    \includegraphics[width=\textwidth]{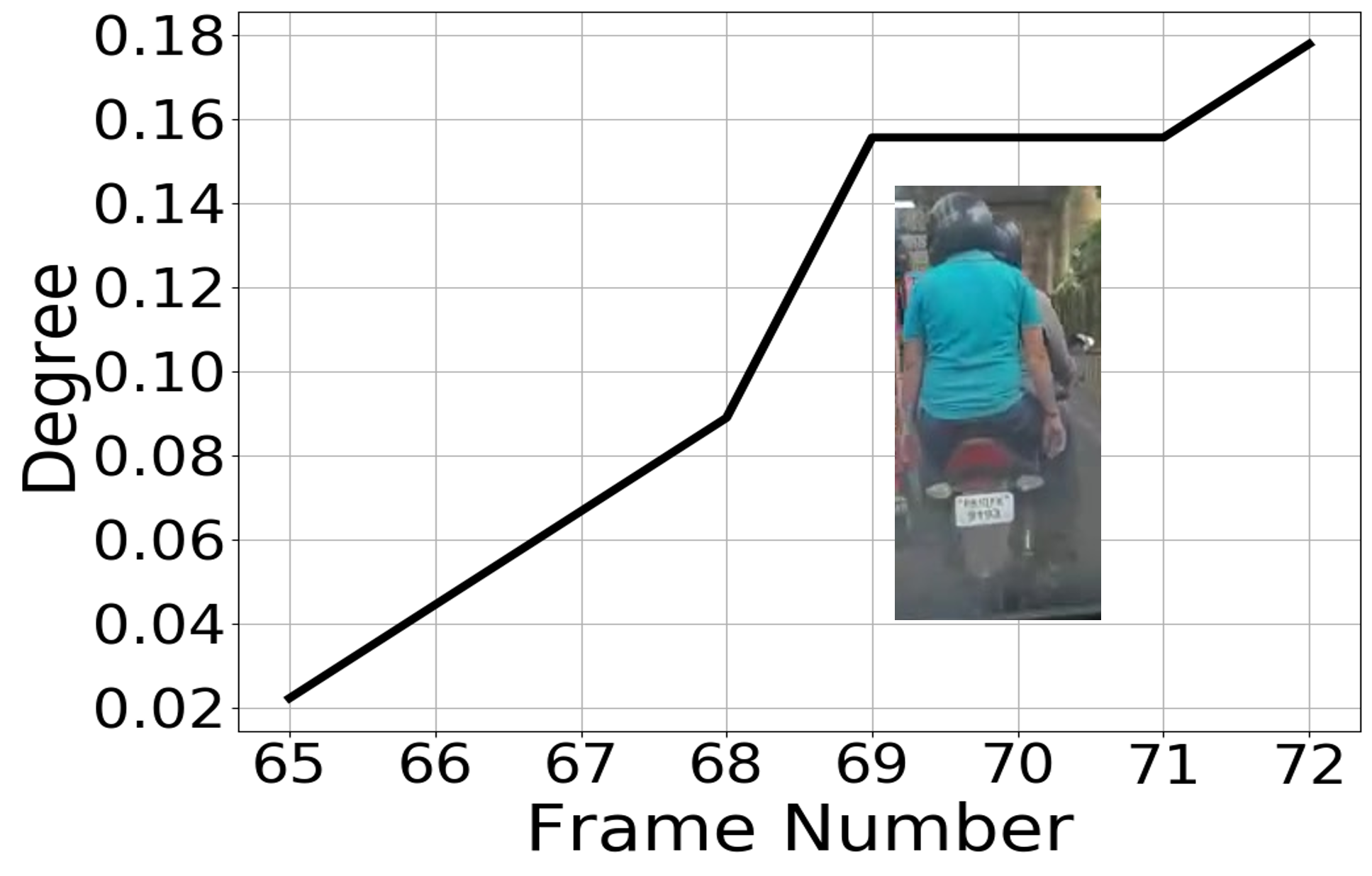}
    \caption{Overspeeding.}
    \label{fig: argo4}
  \end{subfigure}
\caption{\textbf{Driver Behavior Modeling in Singapore \textit{(top row)}, U.S. \textit{(second row)}, China \textit{(third row)}, and India \textit{(bottom row)}:} In each row, the first three figures demonstrate the trajectory of a vehicle executing an aggressive driving style (sudden lane change, overspeeding, weaving, and overspeeding, respectively), while the fourth figure shows the corresponding closeness or degree centrality plot. The shaded colored regions overlaid on the graphs in the first two rows are color heat maps that correspond to $\mathcal{P}(T)$ (line $8$, Algorithm~\ref{alg: ex}). \textbf{Conclusion:} The expected time frame of a driving style reported by the participants of the user study (\textbf{\textcolor{orange}{orange}} dashed line) matches that of the time of maximum likelihood computed by StylePredict (\textbf{\textcolor{blue}{blue}} dashed line).}
  \label{fig: qualitative}
  \vspace{-10pt}
\end{figure*}

\section{Experiments and Results}
\label{sec: experiments_and_results}
We begin with a discussion of the evaluation metric, the Time Deviation Error (TDE), for measuring the accuracy of behavior prediction methods in Section~\ref{subsec: TDE_metric}. Then, we describe the real-world traffic datasets and simulation environment used for testing our approach and outline the annotation algorithm used to generate ground-truth labels for aggressive and conservative vehicles in Section~\ref{subsec: datasets}. Finally, we use the TDE to evaluate our approach and analyze the results in real-world traffic as well as simulation in Section~\ref{subsec: real_world_traffic_analysis}.

\subsection{Evaluation Metric}
\label{subsec: TDE_metric}
We use Time Deviation Error (TDE)~\cite{cmetric}, to measure the temporal difference between a human prediction and a model prediction. For example, if a vehicle executes a rash overtake at the $5^\textrm{th}$ frame and our model predicts the behavior at the $7^\textrm{th}$ frame, then the TDE$=0.067$ seconds, assuming $30$ frames per second. A lower value for TDE indicates a more accurate behavior prediction model. The TDE is given by the following equation,

\begin{equation}
    \textrm{TDE}_\textrm{style} =   \abs*{\frac{t_\textrm{SLE} - \EX[T]}{f}} 
    \label{eq: TDE}
\end{equation}

\noindent where $\EX$ denotes the expected time-stamp of an exhibited behavior in the ground-truth annotated by a human and $f$ is the frame rate of the video. $f=2$ Hz for the Singapore dataset and $f=10$ Hz for the U.S. dataset. In other words, the $\textrm{TDE}_\textrm{style}$ computes the time difference between the mean frame, $\EX[T]$, reported by the participants and the frame with the maximum likelihood, $t_\textrm{SLE}$, predicted by our approach. While $t_\textrm{SLE}$ is computed using $\argmax_{t \in \Delta t}{\textrm{SLE}(t)}$ as explained in Section~\ref{sec: approach}, $\EX[T]$ is computed using Algorithm~\ref{alg: ex}, described in the following section.

A natural question that may be asked is why not use classification accuracy as a metric -- that is, to measure the number of correctly modeled styles as a fraction of the total number of styles. The answer is that since the model is rule-based using the representations and definitions in Sections~\ref{subsec: DGG} and~\ref{sec: centrality}, the classification accuracy metric (and variants thereof) always turns out to be $100\%$. We verified this in our experiments.

\subsection{Datasets and Simulation Environment}
\label{subsec: datasets}

\paragraph{Simulation Environment} 

We use the Highway-Env simulator~\cite{leurent2019social} developed using PyGame~\cite{pygame}. The simulator consists of a 2D environment where vehicles are made to drive along a multi-lane highway using the Bicycle Kinematic Model~\cite{polack2017kinematic} as the underlying motion model. The linear acceleration model is based on the Intelligent Driver Model (IDM)~\cite{treiber2000congested}, while the lane changing behavior is based on the MOBIL~\cite{kesting2007general} model.

The original simulator proposed by Leurent et al.~\cite{leurent2019social} generates homogeneous agents that are parameterized by default to behave conservatively. We modified the simulation parameters and designed two different classes of heterogeneity (see Figure~\ref{fig: simulator_figs}) to produce both conservative (blue agents) and aggressive vehicles (green agent). The parameters used to generate the two classes of vehicles are provided in the supplementary material.


\paragraph{Real-World Datasets}

We have evaluated StylePredict on traffic data collected from geographically diverse regions of the world. In particular, we use data collected in Pittsburgh (U.S.A)~\cite{Argoverse}, New Delhi (India)~\cite{chandra2019traphic}, Beijing (China)~\cite{wang2019apolloscape}, and Singapore (private dataset).
The format of the data includes the timestamp, road-agent I.D., road-agent type, the spatial coordinates, and the location. We understand that the characteristics of drivers in a particular city may not reflect similarly in other cities of the same country. Therefore, all results presented in this work correspond to the traffic in the specific city where the dataset is recorded.

\begin{table}[t]
\centering
\caption{We report the Time Deviation Error (TDE) (in seconds (s)) for the following driving styles: Overspeeding (OS), Overtaking (OT), Sudden Lane-Changes (SLC), and Weaving (W). The TDE indicates the absolute difference between the times taken by a human expert and our proposed approach to identify a driving style. Lower is better. \textit{Conclusion:} On average, we find that it is easiest to predict weaving and sudden lane-changes in India. This observation agrees with our cultural analysis in Section~\ref{subsec: real_world_traffic_analysis} where we show that higher heterogeneity (associated with traffic in India) and lack of conformity in lane-driving leads to more prominent weaving and lane-changing behaviors.}

\centering
\begin{tabular}{lcccc} 
\toprule
 Dataset\Bstrut &  \multicolumn{4}{c}{Styles} \\
\cline{2-5}
& OS \Tstrut & OT & SLC & W\\
\midrule
U.S.~\cite{Argoverse} \Tstrut & 0.25s & 0.67s  & 0.23s & 0.26s  \\
Singapore & 0.54s & 0.88s & 1.21s & 1.28s \\
China~\cite{wang2019apolloscape} & 0.74s & 0.44s & 0.39s & 0.23s\\
India~\cite{chandra2019traphic}  & 0.81s & 0.38s & 0.19s & 0.06s\\
\bottomrule
\end{tabular}
\label{tab: accuracy}
\vspace{-15pt}
\end{table}

One of the main issues with these datasets is that they do not contain labels for aggressive and conservative driving behaviors. Therefore, we create the ground-truth driver behavior annotations using Algorithm~\ref{alg: ex}. We directly use the raw trajectory data from these datasets without any pre-processing or filtering step. For each video, the final ground-truth annotation (or label) is the expected value of the frame at which the ego-vehicle is most likely to be executing an aggressive style. This is denoted as $\EX[T]$. The goal for any driver behavior prediction model should be to predict the aggressive style at a time-stamp as close to $\EX[T]$ as possible. The implied difference in the two time-stamps is measured by the TDE metric.


The TDE metric is computed by Equation~\ref{eq: TDE}. Here, $t_\textrm{SLE} =  \argmax_{t \in \Delta t}{\textrm{SLE}(t)}$, as explained in Section~\ref{sec: approach}. We use Algorithm~\ref{alg: ex} for computing $\EX[T]$. For each video, $M$ participants were asked to mark the starting and end frames for the time-period during which a vehicle is observed executing an aggressive maneuver. For each video, we end up with $S = \{s_1, s_2, \ldots, s_M\}$ and $E = \{e_1, e_2, \ldots, e_M\}$ start and end frames, respectively. We extract the overall start and end frame by finding the minimum and maximum value in $S$ and $E$, respectively (lines $1-2$). We denote these values as $s^*$ and $e^*$. Next, we initialize a distinct counter, $c_t$, for each frame $t \in [s^*, e^*]$ (line $3$). We increment a counter $c_t$ by $1$ if $t \in [s_m, e_m]$ (lines $4-7$). The value of the counter $c_t$ is assigned to $\mathcal{P}(T)$ (line $8$). The $\EX[T]$ of $\mathcal{P}(T)$ can then be computed using the standard definition of expectation of a discrete probability mass function (line $10$). Algorithm~\ref{alg: ex} is applied separately for each video in each dataset.



\begin{figure}[t]
\centering
  \includegraphics[width=\linewidth]{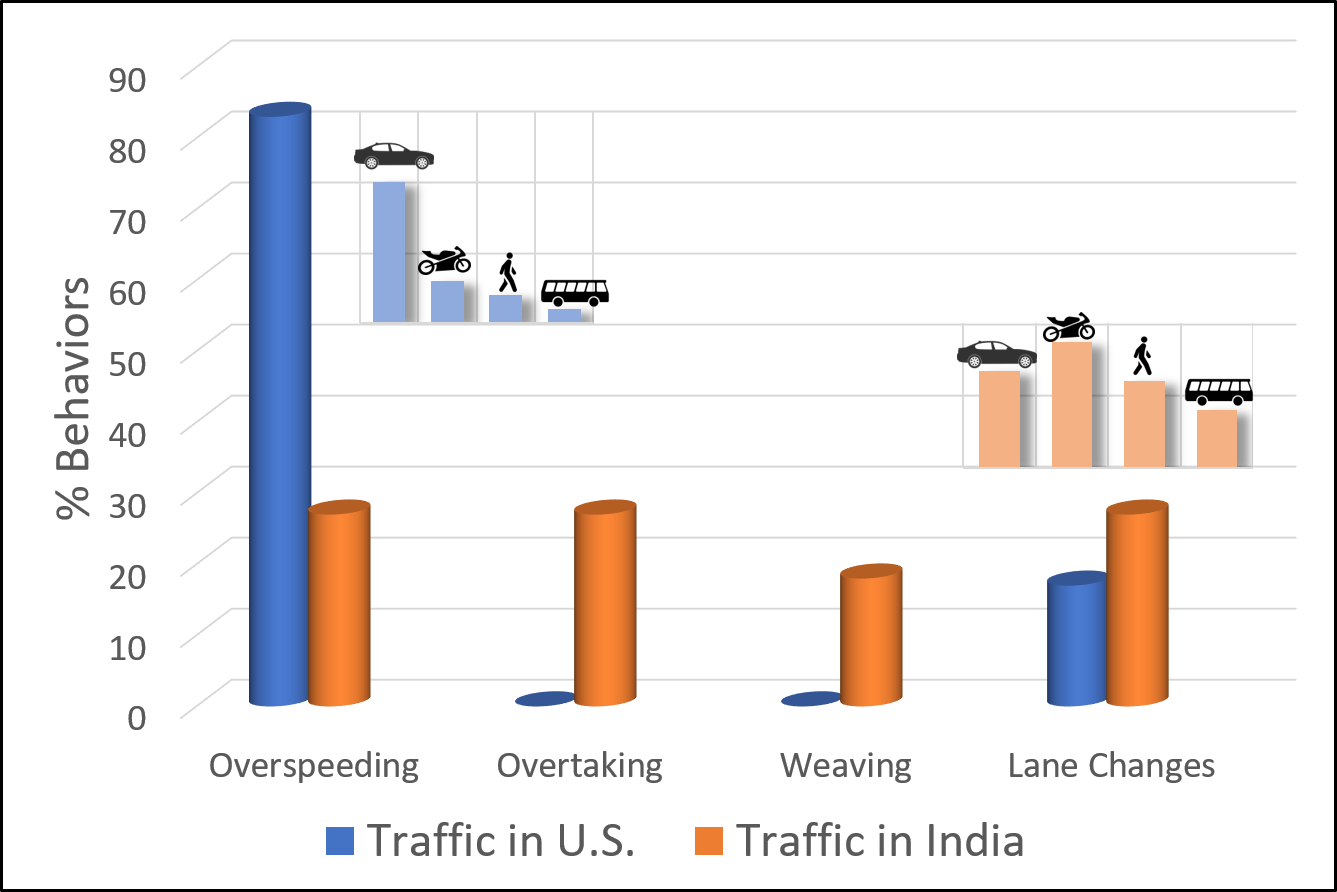}

\caption{Distribution of vehicles (\textit{inner graphs}) and driver behaviors (\textit{outer graph}) in different regions. \textbf{Outer graph conclusion:} The distribution of driving styles is uniform in India and skewed towards longitudinal styles in the U.S. \textbf{Inner graphs conclusion:} The traffic is highly heterogeneous in nature in India whereas the traffic in the U.S. is primarily composed of cars.}

\label{fig: culture2}

\end{figure}

\begin{figure*}[t]
\centering
\begin{subfigure}[h]{.49\linewidth}
  \includegraphics[width=\linewidth]{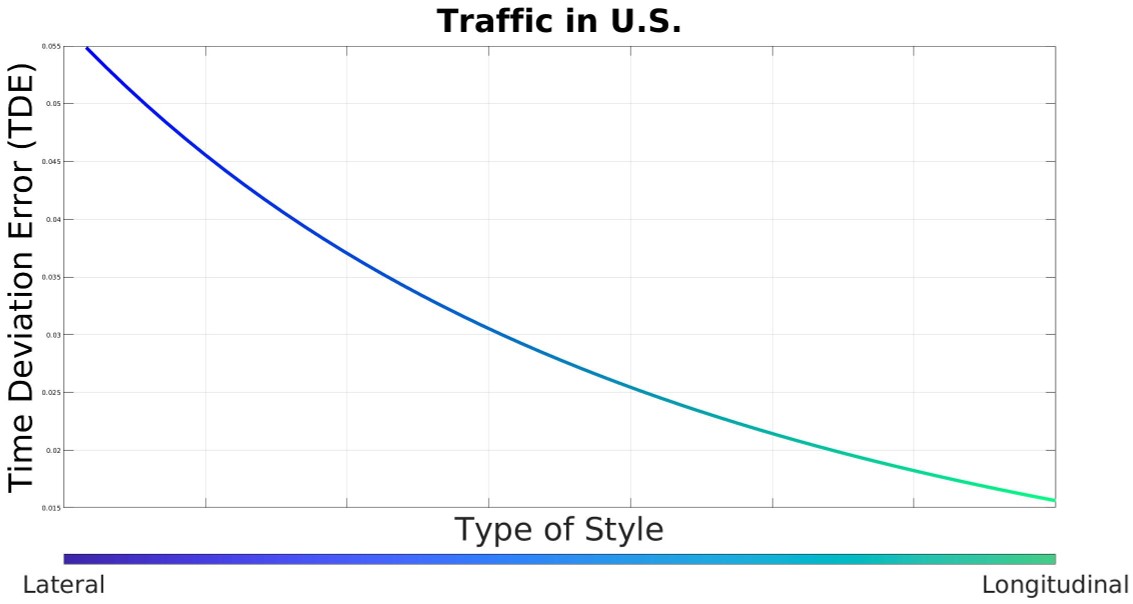}
\caption{Lateral styles (higher TDE) versus longitudinal styles (lower TDE)}
\label{fig: UScurve}
\end{subfigure}
\begin{subfigure}[h]{.49\linewidth}
  \includegraphics[width=\linewidth]{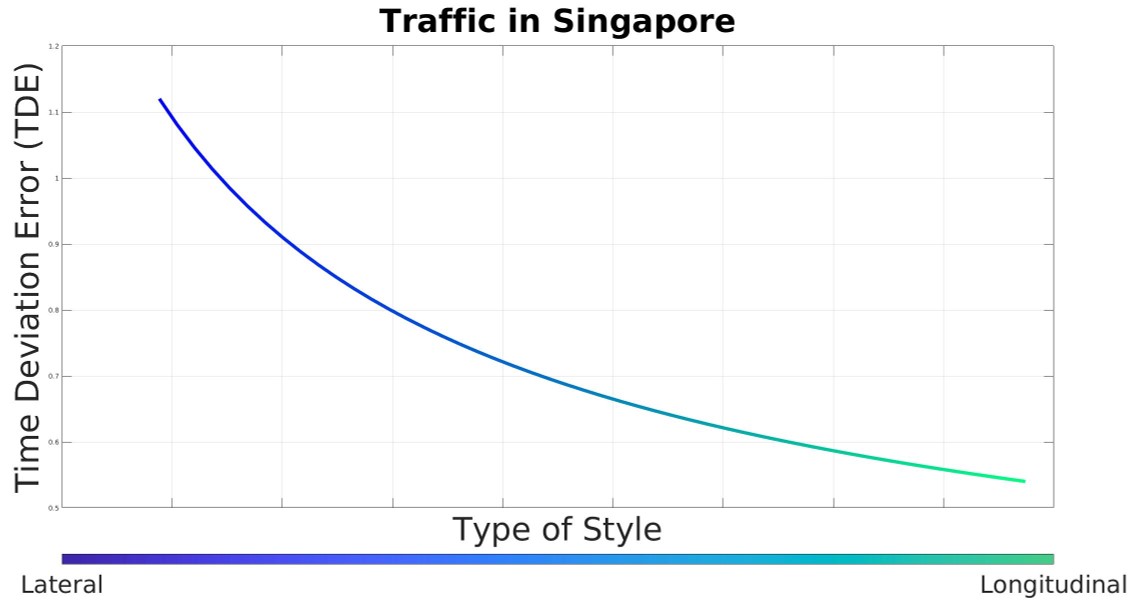}
\caption{Lateral styles (higher TDE) versus longitudinal styles (lower TDE)}
\label{fig: SGcurve}
\end{subfigure}
\begin{subfigure}[h]{.49\linewidth}
  \includegraphics[width=\linewidth]{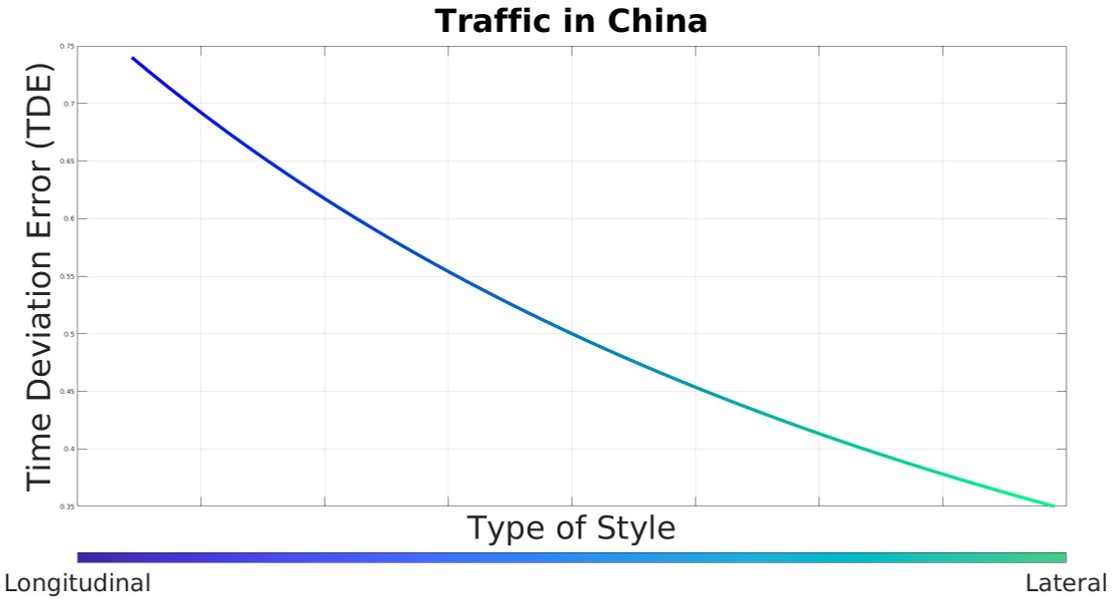}
\caption{Longitudinal styles (higher TDE) versus lateral styles (lower TDE)}
\label{fig: CHcurve}
\end{subfigure}
\begin{subfigure}[h]{.49\linewidth}
  \includegraphics[width=\linewidth]{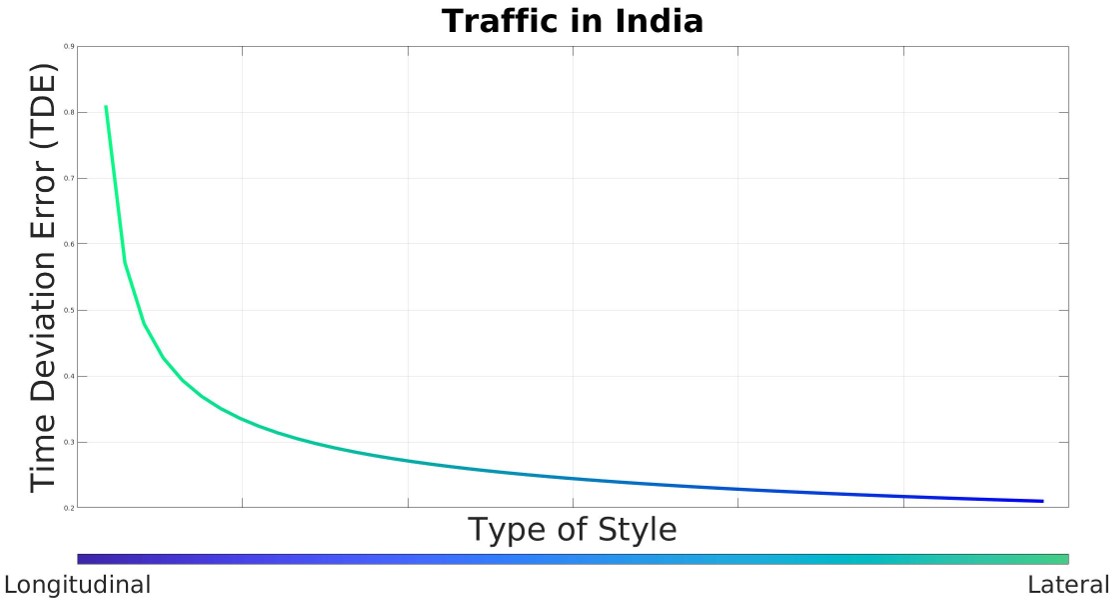}
\caption{Longitudinal styles (higher TDE) versus lateral styles (lower TDE)}
\label{fig: INDcurve}
\end{subfigure}
\caption{\textbf{Inverse Relationship:} We observe an inverse correlation between the TDE and the type of style (longitudinal vs. lateral). The ``type of style'' on the x-axis is a continuous variable. For example, a vehicle may simultaneously overspeed \textit{and} perform a lane-change. The ``type'' of this style would lie somewhere along the middle of the x-axis. (\textit{Top row}) The U.S. and Singapore are countries with low traffic density and lack of heterogeneity, and so drivers are more likely to perform longitudinal driving styles (Section~\ref{sec: culture_analysis}). Therefore StylePredict can model longitudinal styles more accurately (lower TDE) than lateral styles (higher TDE). (\textit{Bottom row}) Conversely, India and China are countries with higher traffic density and heterogeneity, and so drivers are less likely to perform longitudinal driving styles and more likely to perform lateral styles (Section~\ref{sec: culture_analysis}). Therefore, StylePredict can model lateral styles more accurately (lower TDE) than longitudinal styles (higher TDE).  } 
\label{fig: inverse_curve}
\end{figure*}

\subsection{Results using TDE} 
\label{subsec: real_world_traffic_analysis}

In Table~\ref{tab: accuracy}, we report the average TDE in seconds(s) in different geographical regions and cultures for the following driving styles: Overspeeding (OS), Overtaking (OT), Sudden Lane-Changes (SLC), and Weaving (W). 
The traffic conditions differ significantly due to the varying cultural norms in different countries including Singapore, the United States (U.S.), China, and India. For instance, the traffic is more regulated in the U.S. than in Asian countries such as India or China where vehicles do not conform to standard rules such as lane-driving. Such differences contribute to different driving behaviors. Our quantitative results in Table~\ref{tab: accuracy} and qualitative results in Figure~\ref{fig: qualitative} show that our driver behavior modeling algorithm is not affected by cultural norms. Across all cultures, the average TDE is less than $1$ second for every specific style. Aggressive vehicles are still associated with high centrality values while conservative vehicles remain associated with low centrality values.

In Figure~\ref{fig: qualitative}, we show traffic recorded in Singapore \textit{(top row)}, the U.S. \textit{(second row)}, China \textit{(third row)}, and India \textit{(bottom row)}. In each scenario, the first three columns depict the trajectory of a vehicle executing a specific style between some time interval. The last column shows the corresponding centrality plot. The shaded colored regions overlaid on the graphs in Figures~\ref{fig: sg4} and ~\ref{fig: argo4} are color heat maps that correspond to $\mathcal{P}(T)$ (line $8$, Algorithm~\ref{alg: ex}). The orange dashed line indicates the mean time frame, $\EX[T]$, whereas the blue dashed line indicates $t_\textrm{SLE}$. The main result can be observed by noting the negligible distance between the two dashed lines, \textit{i.e.} the TDE.

In the first row (corresponding to traffic in Singapore), for instance, our approach accurately predicts a maximum likelihood of a sudden lane-change by the white sedan at around the $75^\textrm{th}$ frame (blue dashed line, Figure~\ref{fig: sg4}), with an average TDE of $0.88$ seconds. And similarly in the second row (corresponding to traffic in the U.S.), we precisely predict the maximum likelihood of overspeeding by the vehicle denoted by the red dot at around the $30^\textrm{th}$ frame with a TDE of $0.25$ seconds. Note that in both cases the TDE (the distance between the blue and the orange dashed vertical lines) is almost negligible. 

\section{Analysis of Behavior Modeling in Different Cultures}
\label{sec: culture_analysis}

From the results of the experiments performed in the previous section, we draw several novel conclusions highlighting the relationship between driver behavior and traffic environments in the USA, China, India, and Singapore. Specifically, we observe that the traffic density and heterogeneity of a region influences the driving behaviors of road-agents. We summarize our conclusions as follows:


    

\begin{conc}
Drivers in regions with \textit{lower} traffic density and heterogeneity (U.S.A/Singapore) are more likely to perform \textit{longitudinal} driving styles (Figures~\ref{fig: UScurve} and~\ref{fig: SGcurve}). 
\end{conc}

A lower traffic density logically implies larger spaces in which vehicles can navigate, thereby theoretically increasing the chances of overspeeding and underspeeding. This phenomenon is reflected in the TDE measurement. The degree centrality should ideally be computed after observing an agent for some number of time-steps, due to the recursive nature of Definition~\ref{eq: degree}. Long-term observation is easier in sparse traffic and therefore, the TDE for overspeeding is lower for sparser populated countries like the U.S.($0.25$s) and Singapore($0.54$s), as shown in Table~\ref{tab: accuracy}.

\begin{conc}
Drivers in regions with \textit{higher} traffic density and heterogeneity (India/China) are more likely to perform \textit{lateral} driving styles (Figures~\ref{fig: CHcurve} and~\ref{fig: INDcurve}). 
\end{conc}

Conversely,  a higher traffic density implies smaller spaces in which vehicles can navigate, thereby minimizing the likelihood of overspeeding. In addition, countries like India and China are often composed of different road-agents including cars, buses, trucks, pedestrians, three-wheelers and two-wheelers. Two-wheelers and three-wheelers are known to be more likely to perform lateral driving styles~\cite{india-lane}. Therefore, StylePredict achieves the lowest TDE for lateral behaviors in India ($0.06$s for weaving, $0.19$s for sudden lane-changes, and $0.38$s for overtaking) and China ($0.23$s for weaving, $0.39$s for sudden lane-changes, and $0.44$s for overtaking).
    
Due to the influence of heterogeneity on lateral driving styles, the distribution of specific styles executed by vehicles is more uniform in India and China and skewed towards longitudinal styles in the U.S. In Figure~\ref{fig: culture2}, we show that over $83\%$ of aggressive maneuvers in homogeneous U.S. traffic were classified as overspeeding while the remaining $17\%$ were classified as sudden lane-changing. On the other hand, specific styles in traffic in India, which is relatively more heterogeneous than traffic in U.S., were observed to be evenly distributed between $27\%$ sudden lane-changing, $27\%$ overtaking, $27\%$ overspeeding, and $18\%$ weaving.

Based on these conclusions, we observe an inverse correlation between the TDE and the type of style (Figure~\ref{fig: inverse_curve}). For example, traffic in the U.S. and Singapore is relatively sparse and homogeneous. Therefore, by means of the above conclusions, we observe a lower TDE for longitudinal styles and higher TDE for lateral styles. On the other hand, we observe the opposite trend in Asian countries like India and China where traffic is dense and heterogeneous. In such environments, we observe a higher TDE for longitudinal styles and lower TDE for lateral styles.

\section{Conclusions, Limitations, and Future Work}
\label{sec: conclusion}
We present a new Machine Theory of Mind approach that uses the idea of vertex centrality from computational graph theory to explicitly and exactly model the behavior of human drivers in realtime traffic using only the trajectories of the vehicles in the global coordinate frame. Our approach is noise-invariant, can be integrated into any realtime autonomous driving system, and can work in any geographic region. 

We also study the behaviors of drivers in different regions with varying traffic density and heterogeneity and observe an inverse relationship between the longitudinal and lateral driving styles in a specific region. While our approach currently limits itself to modeling driving styles, it can be extended to decision-making and planning, which is a topic of future work.

\section*{Funding Acknowledgement}

This work was supported in part by ARO Grants W911NF1910069 and W911NF1910315, Semiconductor Research Corporation (SRC), and Intel.

\section*{Availability of Data}

The simulator used to generate the simulation results is publicly available~\cite{leurent2019social}. The datasets corresponding to the traffic in the U.S.~\cite{Argoverse}, India~\cite{chandra2019traphic}, and China\cite{wang2019apolloscape} are available publicly, while the dataset corresponding to the traffic in Singapore is private.

\bibliography{pnas-sample}
\clearpage
\section*{S1: Proof to Theorem~\ref{thm: noiseless}}

In Section~\ref{subsec: polynomial_regression}, our algorithm assumes perfect sensor measurements of the global coordinates of all vehicles. However, in real-world systems, even state-of-the-art methods for vehicle localization incur some error in measurements. We consider the case when the raw sensor data is corrupted by some noise $\epsilon$. Without loss of generality, we prove robustness to noise for the degree centrality and the analysis can be extended to other centrality functions. The discrete-valued centrality vector for the $i^\textrm{th}$ agent is given by $\zeta^i \in \mathbb{R}^{T\times 1}$. So $\zeta^1[2]$ corresponds to the degree centrality value of the $1^\textrm{st}$ agent at $t=2$. 

In Section~\ref{subsec: polynomial_regression}, we show that a noiseless estimator may be obtained by solving an ordinary least squares (OLS) system given by Equation~\ref{eq: noiseless_OLS}.  However, in the presence of noise $\epsilon$, the OLS system described in Equation~\ref{eq: noiseless_OLS} is modified as shown below:

\begin{equation}
    \begin{split}
        &M \tilde\beta = \tilde\zeta^i  \\
        &\tilde\beta = \inv{(M^\top M)}M^\top \tilde\zeta^i
    \end{split}
    \label{eq: noisy_OLS}
\end{equation}

\noindent where $\tilde \zeta^i = \zeta^i + \epsilon$. Then we can prove the following,
\begin{theorem}
$\Vts{\tilde \beta - \beta} =  \bigO{\epsilon}$.
\end{theorem}

\begin{proof}
$M$ is $T\times(d+1)$ Vandermonde matrix where $d \ll T$ is the degree of the resulting centrality polynomial. Vandermonde matrices are known to be ill-conditioned with high condition number $\kappa = \frac{\sigma_\textrm{max} \Bstrutfrac}{\sigma_\textrm{min}}$ that increases exponentially with time $T$. From the noisy system given by Equation~\ref{eq: noisy_OLS}, we have,

\begin{equation}
    \begin{split}
        \tilde\beta &= \inv{(M^\top M)} M^\top\tilde\zeta^i \\
        \tilde\beta &= \inv{(M^\top M)} M^\top(\zeta^i + \epsilon) \\
        \tilde\beta &= \beta + \inv{(M^\top M)}M^\top\epsilon \\
    \end{split}
    \label{eq: beta_betatilde_relationship}
\end{equation}

\noindent From Equation~\ref{eq: beta_betatilde_relationship}, $\Vts{\tilde \beta - \beta} = \Vts{\inv{(M^\top M)}M\epsilon}$ which can be shown to be approximately in the order $ \bigO{\kappa\epsilon}$. Therefore, the error between the true solution $\beta$ and the estimated solution in the presence of noise $\tilde \beta$, depends on the condition number $\kappa$ of the matrix $M$. A higher value of $\kappa$ implies that the trailing singular values of $M^\top M$, denoted by $\Sigma = \{ \sigma_1^2, \sigma_2^2, \ldots, \sigma_d^2 \}$ have very small magnitudes. When inverting the matrix $M^\top M$, as in Equation~\ref{eq: beta_betatilde_relationship}, the singular values are inverted (by taking the reciprocal) and are represented by $\inv{\Sigma} = \{ \inv{(\sigma_1^2)}, \inv{(\sigma_2^2)}, \ldots, \inv{(\sigma_d^2)} \}$. After the inversion, the trailing singular values now have large magnitudes, $\Vts{\inv{(\sigma_i^2)}} \gg 1$. When multiplied by $\epsilon$, these large inverted singular values amplify the error, resulting in a large value of $\Vts{\tilde \beta - \beta}$. In Figure~\ref{fig: condition_number}, it can be seen by the red curve that under general conditions, $\kappa$ increases exponentially for even small matrices (we fix $d=2$ and increase $T$ from $0$ to $20$). 

However, there are many techniques that bound the condition number of a matrix by regularizing its singular values. In our application, we use the well-known Tikhonov regularization~\cite{calvetti2003tikhonov}. Under this regularization, after the inversion operation is applied on $M^\top M$, the magnitude of the resulting inverted singular values are constrained by adding a parameter $\alpha$. The modified inverted singular values can be expressed as $\frac{\sigma_i \Bstrutfrac}{\Tstrutfrac \sigma_i^2 + \alpha^2}$. The addition of the $\alpha^2$ in the denominator keeps the overall magnitude of the inverted singular value from ``blowing up''. The choice of the parameter $\alpha$ is important and a detailed analysis of computing $\alpha$ is given in the supplementary material. In our approach, as $M$ is fixed for every $T$, we need only search for the optimal $\alpha$ once for every $T$.

The new estimator is therefore obtained by regularizing Equation~\ref{eq: noisy_OLS} by adding a multiple of the Identity matrix, $\Gamma = \alpha I$ to $M^\top M$,

\[ \tilde\beta = \beta + \inv{(M^\top M + \Gamma^\top \Gamma)}M^\top\epsilon \]

\noindent The effect of this regularization on the condition number of $M$ can be seen by the blue curve in Figure~\ref{fig: condition_number}. 

\end{proof}
\begin{figure}[t]
    \centering
    \includegraphics[width = \columnwidth]{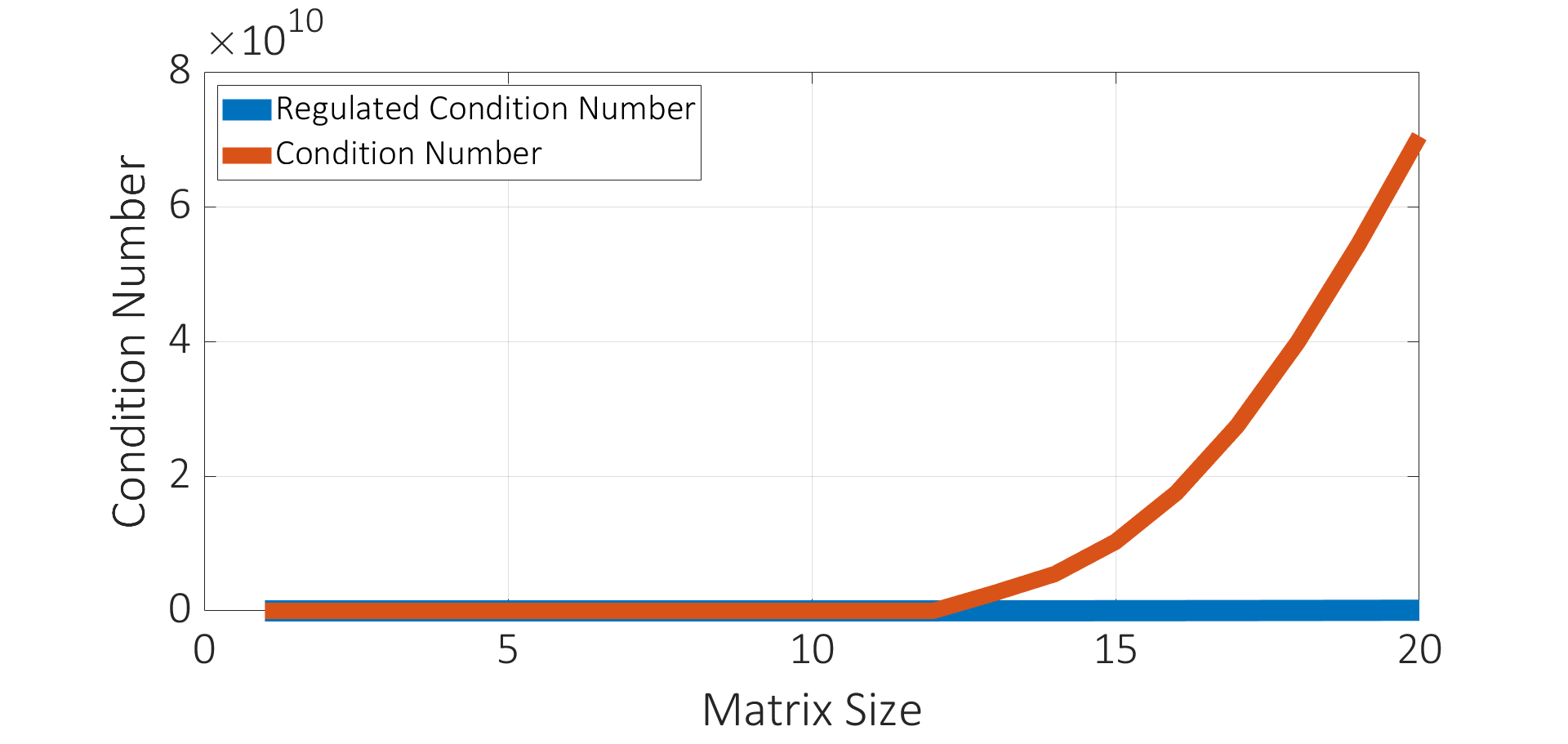}
    \caption{\textbf{Robustness to Noise:} We show that by regularizing the noisy OLS system given by Equation~\ref{eq: noisy_OLS}, we can reduce the original condition number (red curve) while at the same time upper bounding the reduced condition number (blue curve) by $\delta \longrightarrow 0$. The reduced condition number helps stabilize the noisy estimator $\tilde \beta$. }
    \label{fig: condition_number}
\end{figure}
\section*{S2: Simulator Parameters}

The linear acceleration model is based on the Intelligent Driver Model (IDM)~\cite{treiber2000congested} and is computed via the following kinematic equation,
\begin{equation}
    \dot v_{\alpha} = a\begin{bmatrix}1 - (\frac{v_{\alpha}}{v_0^{\alpha}})^4 - (\frac{s^*(v_{\alpha}, \Delta v_{\alpha})}{s_{\alpha}})^2\end{bmatrix}
    \label{eq: IDM_acc}
\end{equation}

\noindent Here, the linear acceleration, $\dot v_{\alpha}$, is a function of the velocity $v_{\alpha}$, the net distance gap $s_{\alpha}$ and the velocity difference $\Delta v_{\alpha}$ between the ego-vehicle and the vehicle in front. Equation~\ref{eq: IDM_acc} is a combination of the acceleration on a free road $\dot v_{free} = a[1 - (v/v_0)^{4}]$ (\textit{i.e.} no obstacles) and the braking deceleration, $-a(s^*(v_\alpha,\Delta v_\alpha)/s_\alpha)^2$ (\textit{i.e.} when the ego-vehicle comes in close proximity to the vehicle in front). The deceleration term depends on the ratio of the desired minimum gap ($s^*(v_\alpha,\Delta v_\alpha)$) and the actual gap ($s_{\alpha}$), where $s^* (v_\alpha,\Delta v_\alpha)= s_0 + vT + \frac{v\Delta v}{2\sqrt{ab}}$. $s_0$ is the minimum distance in congested traffic, $vT$ is the distance while following the leading vehicle at a constant safety time gap $T$, and $a,b$ correspond to the comfortable maximum acceleration and comfortable maximum deceleration, respectively.

The lane changing behavior is based on the MOBIL~\cite{kesting2007general} model. According to this model, there are two key parameters when considering a lane-change:
\begin{enumerate}
    \item \textit{Safety Criterion}: This condition checks if, after a lane-change to a target lane, the ego-vehicle has enough room to accelerate. Formally, we check if the deceleration of the successor $a_\textrm{target}$ in the target lane exceeds a pre-defined safe limit $b_{safe}$:
    \begin{equation*}
        a_\textrm{target} \geq -b_{safe}
    \end{equation*}
    
    \item \textit{Incentive Criterion}: This criterion determines the total advantage to the ego-vehicle after the lane-change, measured in terms of total acceleration gain or loss. It is computed with the formula,
    
    \begin{equation*}
    \tilde{a}_\textrm{ego} - a_\textrm{ego} + p(\tilde{a}_n - a_n + \tilde{a}_o - a_o) > \Delta a_{th}
    \end{equation*}
    where $\tilde{a}_\textrm{ego} - a_\textrm{ego}$ represents the acceleration gain that the ego-vehicle would receive after to the lane change. The second term denotes the total acceleration gain/loss of the immediate neighbours (the new follower in the target, $a_n$, and the original follower in the current lane, $a_o$) weighted with the politeness factor, $p$. By adjusting $p$ the intent of the drivers can be changed from purely egoistic ($p=0$) to more altruistic ($p=1$). We refer the reader to~\cite{kesting2007general} for further details.
\end{enumerate}

\noindent The lane change is executed if both the safety criterion is satisfied, \textit{and} the total acceleration gain is more than the defined minimum acceleration gain, $\Delta a_{th}$.

\begin{table}[h]
\caption{We show the simulation parameters that define the conservative and aggressive vehicle classes.}
\centering
\resizebox{\columnwidth}{!}{%
\begin{tabular}{cccc} 
\toprule
Model & Parameter \Tstrut & Conservative \Bstrut &   Aggressive \\
\hline
\multirow{4}{*}{IDM}& Time gap ( $T$) \Tstrut & 1.5s      & 1.2s \\
 &Min distance ($s_0$) & 5.0 $m$ & 2.5 $m$ \\
&Max comfort acc. ($a$)     & 3.0 $m/s^2$ & 6.0 $m/s^2$\\
&Max comfort dec. ($b$) & 6.0 $m/s^2$               &  9.0 $m/s^2$ \\
\midrule
\multirow{3}{*}{MOBIL}& Politeness ($p$) & 0.5     & 0\\
& Min acc gain ($\Delta a_{th}$) & 0.2 $m/s^2$ & 0 $m/s^2$ \\
& Safe acc limit ($b_{safe}$) & 3.0 $m/s^2$ & 9.0 $m/s^2$\\
\bottomrule
\end{tabular}
}
\label{tab: parameters}
\vspace{-10pt}
\end{table}

Additionally, the desired velocity $v_0$ is set to $25$ meters per second and $40$ meters per second for the conservative and aggressive vehicle classes, respectively. Finally, the desired velocities for the conservative vehicles were uniformly distributed with a variation of {$\pm$10\%} to increase the heterogeneity in the simulation environment.

\end{document}